\documentclass[runningheads]{llncs}

\usepackage[utf8]{inputenc}
\usepackage[T1]{fontenc}
\usepackage[leqno]{amsmath}
\usepackage{amssymb}
\usepackage{stmaryrd}
\usepackage[unicode,pdfmenubar,linktoc=all,hidelinks,bookmarks, pdftex]{hyperref}
\usepackage{tikz}
\usetikzlibrary{fadings}
\usetikzlibrary{shapes,decorations,positioning}
\usetikzlibrary{fadings,shapes,decorations,positioning,calc}
\usetikzlibrary{decorations.markings,decorations.pathreplacing}
\usetikzlibrary{shapes.geometric,automata,positioning}
\usetikzlibrary{shadows}\usetikzlibrary{calc}
\usetikzlibrary{decorations,decorations.pathmorphing,decorations.pathreplacing}
\usetikzlibrary{calc}
\usepackage{cite}

\DeclareFontFamily{U}{mathb}{\hyphenchar\font45}
\DeclareFontShape{U}{mathb}{m}{n}{
	<-6> mathb5 <6-7> mathb6 <7-8> mathb7
	<8-9> mathb8 <9-10> mathb9
	<10-12> mathb10 <12-> mathb12
}{}
\DeclareSymbolFont{mathbA}{U}{mathb}{m}{it}
\usepackage{fonttable}
\DeclareMathSymbol{\llcurly}{\mathrel}{mathbA}{"CE}
\DeclareMathSymbol{\ggcurly}{\mathrel}{mathbA}{"CF}
\DeclareMathSymbol{\square}{\mathrel}{mathbA}{5}
\DeclareMathSymbol{\centerdot}{\mathrel}{mathbA}{13}

\newcommand{\change}{\ensuremath{\circ}}

\usepackage{xspace}
\usepackage{graphicx}
\usepackage{booktabs}
\usepackage{arydshln}
\usepackage{enumerate}
\usepackage{csquotes}
\usepackage{xparse}
\usepackage{wasysym}

\input{macros_sauerwald.tex}

\makeatletter
\def\blfootnote{\gdef\@thefnmark{}\@footnotetext}
\makeatother

\renewcommand{\beliefsOf}[1]{\ensuremath{\text{Bel}\left(#1\right)}}

\def\myorcidID#1{\ifx\hyper@anchor\@undefined\unskip$^{[#1]}$\else\href{#1}{\protect\includegraphics[height=8px]{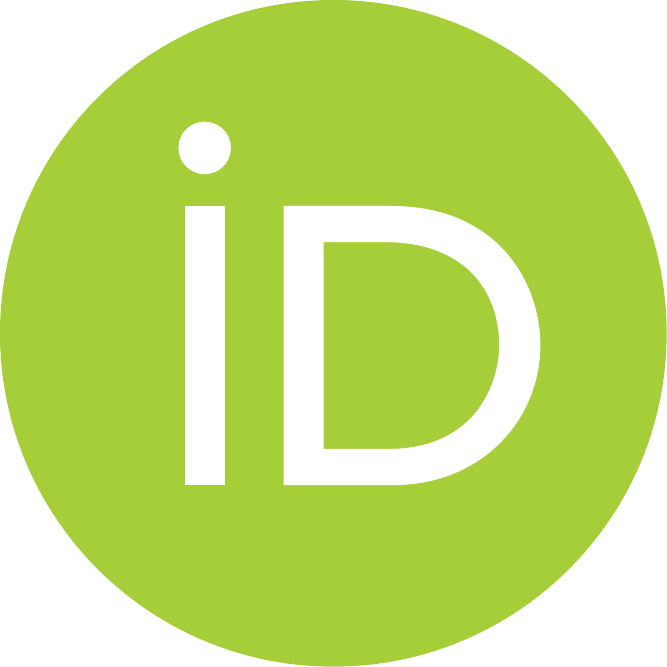}}\fi}

\renewcommand{\headword}[1]{\emph{#1}}

\newcommand{\hesitant}{hesitant}
\newcommand{\decrement}{decrement}

\newcommand{\improve}{\ensuremath{\square}}
\newcommand{\improveLimit}{\ensuremath{\centerdot}}

\newcommand{\stepwise}{\mathrel{\circ}}

\newcommand{\belated}{\mathrel{\circ}}

\begin{document}
%
\title{Decrement Operators in Belief Change}

\author{Kai Sauerwald
	\href{https://orcid.org/0000-0002-1551-7016}{\protect\includegraphics[height=8px]{orcid_logo.pdf}}%
		\and Christoph Beierle
}
\authorrunning{Sauerwald and Beierle}
\institute{FernUniversit\"{a}t in Hagen, 58084 Hagen, Germany\\\email{\{kai.sauerwald,christoph.beierle\}@fernuni-hagen.de}
}

%
\maketitle

%
\begin{abstract}
While research on iterated revision is predominant in the field of iterated belief change, the class of iterated contraction operators
received more attention in recent years.
In this article, we examine a non-prioritized generalisation of iterated contraction. 
In particular,  the class of weak \decrement\ operators is introduced, 
which are operators that by multiple steps achieve the same as a contraction.
Inspired by Darwiche and Pearl's work on iterated revision the subclass of \decrement\ operators is defined. 
For both, decrement and weak decrement operators, postulates are presented and for each of them a representation theorem in the framework of total preorders is given.
Furthermore, we present two sub-types of \decrement\ operators.

\end{abstract}

\section{Introduction}
\label{sec:introduction}
Changing beliefs in a rational way in the light of new information is one of the core abilities of an agent - and thus one of the main concerns of artificial intelligence. 
The established AGM theory \cite{KS_AlchourronGaerdenforsMakinson1985} deals with desirable properties of rational belief change. 
	The AGM approach provides properties for different types of belief changes. If new beliefs are incorporated into an agent's beliefs while maintaining consistency, this is called a revision. 
	Expansion adds a belief unquestioned to an agent's beliefs, and  contraction removes a belief from an agent's beliefs.
	Building upon the characterisations of these kinds of changes and the underlying principle of minimal change, the theory fanned out in different directions and sub-fields.

The field of iterated belief revision examines the properties of belief revision operators which, due to their nature, can be applied iteratively. 
In this sub-field, one of the most influential articles is the seminal paper \cite{KS_DarwichePearl1997} by Darwiche and Pearl (DP), establishing the insight that belief sets are not a sufficient representation for iterated belief revision.
An agent has to encode more information about her belief change strategy into her \emph{epistemic state} - where the revision strategy deeply corresponds with conditional beliefs.
This requires additional postulates that guarantee intended behaviour in forthcoming changes.
The common way of encoding, also established by Darwiche and Pearl \cite{KS_DarwichePearl1997}, is an extension of Katsuno and Mendelzon's characterisation of AGM revision in terms of plausibility orderings \cite{KS_KatsunoMendelzon1992}, where it is assumed that the epistemic states contain an order over worlds (or interpretations).

Similar work has been done in recent years for iterated contraction. 
Chopra, Ghose, Meyer and Wong \cite{KS_ChopraGhoseMeyerWong2008} contributed postulates for contraction on epistemic states. 
Caridroit, Konieczny and Marquis \cite{KS_CaridroitKoniecznyMarquis2015} provided postulates for contraction in propositional logic and a characterisation with plausibility orders in the style of Katsuno and Mendelzon.
By this characterisation, the main characteristic of a contraction  with $ \alpha $ is  that the worlds of the previous state remain plausible and that the most plausible counter-models of $ \alpha $ become plausible.

However, in the sub-field of non-prioritised belief change, or more specifically, in the field of gradual belief change much work remains to be done on contraction.
An important generalisation of iterated revision operators are the class of improvement operators by Konieczny and {Pino P{\'{e}}rez} \cite{KS_KoniecznyPinoPerez2008}, which achieve the state of an revision by multiple steps in a gradual way.
These kind of changes where intensively studied by Konieczny, {Pino P{\'{e}}rez}, Booth, Ferm{\'{e}} and Grespan \cite{KS_KoniecznyGrespanPinoPerez2010,KS_BoothFermeKoniecznyPerez2014}.
A counterpart of improvement operators for the case of contraction is missing. This article fills this gap.
We investigate the contraction analogon to  improvement operators, which we call decrement operators. 
The leading idea is to examine a class of operators which lead, after enough consecutive applications, to the same states as an (iterative) contraction would do.

The research presented in this paper is also motivated by the quest for a formalisation of forgetting operators within the field of knowledge representation and reasoning (KRR). 
		In a recent survey article by Eiter and Kern-Isberner \cite{EiterKernIsberner2019KIzeitschrift} the connection between contraction and forgetting of a belief is dealt with from a KRR point of view.
		Steps towards a general framework for kinds of forgetting in common-sense based belief management, revealing links to well-known KRR methods, are taken in \cite{BeierleKernIsbernerSauerwaldBockRagni2019KIzeitschrift}.
		However, for the fading out of rarely used beliefs that takes places in humans gradually over time, or
		for the change of routines, e.g. in established workflows, often requiring many iterations and the intentional forgetting of the previous routines,
		counterparts in the formal methods of KRR are missing.
		With our work on \decrement\ operators, we provide some basic building blocks that may prove
		useful for developing a formalisation of these 
			psychologically inspired 
			forgetting operations.

\noindent In summary, the main contributions of this paper are\footnote{This version of the paper contains the full proofs. 
}:
\vspace{-0.15cm}
\begin{itemize}
	\item Postulates for operators which allow one to perform contractions gradually.
	\item Representation theorems for these classes in the framework or epistemic states and total preorders.
	\item Define two special types of \decrement\  operators.
\end{itemize}%
\vspace{-0.15cm}
The rest of the paper is organised as follows.
Section \ref{sec:prelim} briefly presents the required background on belief change.
Section \ref{sec:weak_disimprovements} introduces the main idea and the postulates along with a representation theorem for weak \decrement\ operators.
In Section \ref{sec:disimprovements} the weak \decrement\ operators are restricted by DP-like iteration postulates, leading to the class of \decrement\ operators;
we give also a representation theorem for the class of \decrement\ operators.
In  Section \ref{sec:stepwise_decrement_operator} two special types of \decrement\ operators are specified.
We close the paper with a discussion and point out future work in Section \ref{sec:conclusion}.

\section{Background}
\label{sec:prelim}

Let $ \Sigma $ be a propositional signature. 
The propositional language $ \propLang_\Sigma $ is the smallest set, such that $ a\in\propLang_\Sigma $ for every $ a\in\propLang_\Sigma $ and $ \neg \alpha\in\propLang_\Sigma $, $ \alpha\land \beta,\alpha\lor \beta\in\propLang_\Sigma $ if $ \alpha,\beta\in\propLang_\Sigma $. 
We omit often $ \Sigma $ and write $ \propLang $ instead of $ \propLang_\Sigma $.
We write formulas in $ \propLang $ with lower Greek letters $ \alpha,\beta,\gamma,\ldots $, and propositional variables with lower case letters $ a,b,c,\ldots\in\Sigma $.
 The set of 
propositional interpretations $ \Omega $, also called set of worlds, is identified with the set of corresponding complete conjunctions over $ \Sigma $.
Propositional entailment is denoted by $ \models $, with $ \modelsOf{\alpha} $ we denote the set of models of $ \alpha $, and $ Cn(\alpha)=\{ \beta\mid \alpha\models \beta \} $ is the deductive closure of $ \alpha $.
This is lifted to a set $ X $ by defining $ Cn(X)=\{ \beta \mid X\models\beta \} $.
	For two sets of formulas $ X,Y $ we say $ X $ is equivalent to $ Y $ with respect to the formula $ \alpha $, written $ X =_\alpha Y $, if $ Cn(X\cup\{\alpha\}) = Cn(Y\cup\{\alpha\}) $%
	\footnote{$ Cn(X\cup\{\alpha\}) $ matches 
			belief expansion with $ \alpha $ on belief sets. However, in the context here, the context of iterative changes, we understand this purely technically.
			The problem of expansion in this context is more complex \cite{KS_FermeWassermann2018}.}.
	For two sets of interpretations $ \Omega_1,\Omega_2\subseteq\Omega $ we say $ \Omega_1 $ is equivalent to $ \Omega_2 $ with respect to the formula $ \alpha $, written $ \Omega_1 =_\alpha \Omega_2 $, if $ \Omega_1 $ and $ \Omega_2 $ contain the same set of models of $ \alpha $, i.e. $ \{ \omega_1 \in \Omega_1 \mid \omega_1\models\alpha \} = \{ \omega_2 \in \Omega_2 \mid \omega_2\models\alpha \}  $.
For a set of worlds $ \Omega'\subseteq \Omega $ and  a total preorder $ \leq $ (reflexive and transitive relation) over $ \Omega $, we denote with $ \min(\Omega',\leq)=\{ \omega\mid  \omega\in\Omega' \text{ and } \forall \omega'\in\Omega'\ \omega\leq \omega' \} $ the set of all worlds in the lowest layer of $ \leq $ that are elements in $ \Omega' $.
For a total preorder $ \leq $, we denote with $ < $ its strict variant, i.e. $ x < y $ iff $ x \leq y $ and $ y \not\leq x $; with $ \ll $ the direct successor variant, i.e. $ x \ll y $ iff $ x < y $ and there is no $ z $ such that $ x < z < y $; and we write $ x \simeq y $ iff $ x \leq y $ and $ y\leq x $.

\subsection{Epistemic States and Belief Changes}
Every agent is equipped with an \headword{epistemic state}, sometimes also called belief state, that maintains all necessary information for her belief apparatus. With $ \setAllES $ we denote the set of all epistemic states.
Without defining what a epistemic state is, we assume that for every epistemic state $ \Psi\in\setAllES $ we can obtain the set of plausible sentences $ \beliefsOf{\Psi}\subseteq \mathcal{L} $ of $ \Psi $, which is deductively closed.
We write $ \Psi\models\alpha $ iff $ \alpha\in\beliefsOf{\Psi} $ and we define $ \modelsOf{\Psi}=\{ \omega \mid \omega\models \alpha \text{ for each } \alpha\in\beliefsOf{\Psi} \} $.
A 
belief change operator over $\mathcal{L}$ is a (left-associative) function $ \circ : \setAllES \times \propLang \to 
\setAllES $.
We
denote with $ \Psi \circ^n \alpha $ the n-times application of $ \alpha $ by $ \circ $ to $ \Psi $ \cite{KS_KoniecznyPinoPerez2008}.

Darwiche and Pearl \cite{KS_DarwichePearl1997} propose that an epistemic state $ \psi $ should be equipped with an ordering $ \leq_{\Psi} $ of the worlds (interpretations),
where the compatibility with $ \beliefsOf{\Psi} $ is ensured by the so-called faithfulness. 
Based on the work of Katsuno and Medelezon \cite{KS_KatsunoMendelzon1992}, a mapping $ \Psi \mapsto\, \leq_{\Psi} $ is called faithful assignment if the following is satisfied~\cite{KS_DarwichePearl1997}:
\begin{align*}
& \ksIF \omega_1 \in \modelsOf{\Psi} \ksAND \omega_2 \in \modelsOf{\Psi} \ksTHEN \omega_1 \simeq_\Psi \omega_2  \\
& \ksIF \omega_1 \in \modelsOf{\Psi} \ksAND \omega_2\notin\modelsOf{\Psi} \ksTHEN \omega_1 <_\Psi \omega_2
\end{align*}
Konieczny and Pino Pérez give a stronger variant of faithful assignments for iterated belief change \cite{KS_KoniecznyPinoPerez2008}, which ensures
that the mapping $ \Psi\mapsto\leq_{\Psi} $ is compatible with the belief change operator with respect to syntax independence.

\begin{definition}[Strong Faithful Assignment \cite{KS_KoniecznyPinoPerez2008}]
	Let $ \circ $ be a belief change operator. A function $ \Psi\mapsto \leq_\Psi $ that maps each epistemic state to a total preorder on interpretations is said to be a strong faithful assigment with respect to $ \circ $ if:
	\begin{align*}
	& \ksIF \omega_1 \in \modelsOf{\Psi} \ksAND \omega_2 \in \modelsOf{\Psi} \ksTHEN \omega_1 \simeq_\Psi \omega_2 \tag{SFA1} \label{pstl:SFA1} \\
	& \ksIF \omega_1 \in \modelsOf{\Psi} \ksAND \omega_2\notin\modelsOf{\Psi} \ksTHEN \omega_1 <_\Psi \omega_2 \tag{SFA2} \label{pstl:SFA2} \\
	& \ksIF \alpha_1\equiv\beta_1, \ldots, \alpha_n\equiv\beta_n \ksTHEN \leq_{\Psi\circ\alpha_1\circ\ldots\circ\alpha_n} = \leq_{\Psi\circ\beta_1\circ\ldots\circ\beta_n} \tag{SFA3} \label{pstl:SFA3} 
	\end{align*}
\end{definition}
We will make use of strong faithful assignments for the characterisation theorems.

\subsection{Iterated Contraction}

Postulates for AGM contraction in the framework of epistemic states were given by Chopra, Ghose, Meyer and  Wong \cite{KS_ChopraGhoseMeyerWong2008} and by Konieczny and Pino P{\'{e}}rez \cite{KS_KoniecznyPinoPerez2017}.
 We give here the formulation by Chropra et al. \cite{KS_ChopraGhoseMeyerWong2008}:
\begin{align*}
	 & \beliefsOf{\Psi - \alpha} \subseteq \beliefsOf{\Psi}  \tag{C1}\label{pstl:C-1}                                                      \\
& \ksIF  \alpha\notin\beliefsOf{\Psi} \ksTHEN \beliefsOf{\Psi}\subseteq \beliefsOf{\Psi -\alpha} \tag{C2}\label{pstl:C-2}                                                  \\
& \ksIF \alpha \not\equiv \top \ksTHEN  \alpha \notin \beliefsOf{\Psi  -  \alpha} \tag{C3}\label{pstl:C-3}                                            \\
&  \beliefsOf{\Psi} \subseteq Cn(\beliefsOf{\Psi  -  \alpha} \cup \alpha)  \tag{C4}\label{pstl:C-4}                                     \\
	 & \ksIF \alpha \equiv\beta \ksTHEN \beliefsOf{\Psi  -  \alpha} = \beliefsOf{\Psi  -  \beta} \tag{C5}\label{pstl:C-5}                                           \\
	 & \beliefsOf{\Psi - \alpha} \cap \beliefsOf{\Psi - \beta} \subseteq \beliefsOf{\Psi  -  (\alpha\land\beta)} \tag{C6}\label{pstl:C-6}                                        \\
	 & \ksIF \beta\notin\beliefsOf{\Psi  -  (\alpha\land\beta)}  \ksTHEN \beliefsOf{\Psi  -  (\alpha\land\beta)} \subseteq \beliefsOf{\Psi - \beta} \tag{C7}\label{pstl:C-7}
\end{align*}
For an explanation of these postulates we refer to the article of Caridroit et al. \cite{KS_CaridroitKoniecznyMarquis2015}.
A characterisation in terms of total preorders on epistemic states is given by the following proposition.
\begin{proposition}[AGM Contraction for Epistemic State {\cite{KS_KoniecznyPinoPerez2017}}]\label{prop:es_contraction}
A belief change operator $ - $ fulfils the postulates \eqref{pstl:C-1} to \eqref{pstl:C-7} if and only if there is a faithful assignment $ \Psi\mapsto \leq_\Psi $ such that:
\begin{equation}
\modelsOf{\Psi - \alpha} = \modelsOf{\Psi} \cup \min(\modelsOf{\negOf{\alpha}},\leq_{\Psi}) \label{eq:agmes_contraction}
\end{equation}
\end{proposition}
In addition to the postulates \eqref{pstl:C-1} to \eqref{pstl:C-7}, Konieczny and Pino P{\'{e}}rez give DP-like postulates for intended iteration behaviour of contraction \cite{KS_KoniecznyPinoPerez2017}. 
In the following, we call these class of operators iterated contraction operators, which are characterized by the following proposition.

\begin{proposition}[Iterated Contraction{\cite{KS_KoniecznyPinoPerez2017}}]\label{prop:it_es_contraction}
Let $ - $ be a belief change operator $ - $ which satisfies \eqref{pstl:C-1} to \eqref{pstl:C-7}. 
Then $ - $ is an \emph{iterated contraction operator} if and only if there exists a faithful assignment $ \Psi\mapsto\leq_{\Psi} $ such that \eqref{eq:agmes_contraction} holds and the following is satisfied:
\begin{align*}
	& \ksIF \omega_1,\omega_2 \in \modelsOf{\alpha} \ksTHEN \omega_1 \leq_{\Psi} \omega_2 \Leftrightarrow \omega_1 \leq_{\Psi-\alpha} \omega_2                                       \\
& \ksIF \omega_1,\omega_2 \in \modelsOf{\negOf{\alpha}} \ksTHEN \omega_1 \leq_{\Psi} \omega_2 \Leftrightarrow \omega_1 \leq_{\Psi-\alpha} \omega_2                               \\
& \ksIF \omega_1\in\modelsOf{\negOf{\alpha}} \ksAND \omega_2\in\modelsOf{\alpha}  \ksTHEN    \omega_1 <_{\Psi} \omega_2 \Rightarrow \omega_1 <_{\Psi-\alpha} \omega_2  \\
& \ksIF \omega_1\in\modelsOf{\negOf{\alpha}} \ksAND \omega_2\in\modelsOf{\alpha}  \ksTHEN    \omega_1 \leq_{\Psi} \omega_2 \Rightarrow \omega_1 \leq_{\Psi-\alpha} \omega_2        
\end{align*}
\end{proposition}

\subsection{Improvement Operators}
	
The idea of (weak) improvements is to split the process of an AGM revision for epistemic states \cite[p. 7ff]{KS_DarwichePearl1997} into multiple steps of an operator~$ \improve $.
For such a gradual operator $ \improve $ define $ \Psi \improveLimit \alpha = \Psi \improve^n \alpha $, 
	where $ n\in \naturals $ is smallest integer such that $ \alpha\notin\beliefsOf{\Psi \improve^n \alpha} $.
	In the initial paper about improvement operators \cite{KS_KoniecznyPinoPerez2008},  Konieczny and Pino P{\'{e}}rez gave postulates for $ \improve $, such that $ \improveLimit $ is an AGM revision for epistemic states.
	Due to space reasons, we refer the interested reader to the original paper for the postulates \cite{KS_KoniecznyPinoPerez2008}. 
	The following representation theorem gives an impression on weak improvement operators.
\begin{proposition}[Weak Improvement Operator{\cite[Thm. 1]{KS_KoniecznyPinoPerez2008}}]\label{prop:weak_improve}
	A belief change operator $ \improve $ is a weak improvement operator  if and only if there exists a strong faithful assignment $ \Psi\mapsto \leq_{\Psi} $ such that:
	\begin{equation*}
	\modelsOf{\Psi \improveLimit \alpha} = \min(\modelsOf{\alpha},\leq_{\Psi})
	\end{equation*}
\end{proposition}
Furthermore, the class of weak improvement operators is restricted by DP-like iteration postulates to the so-called improvement operators\cite{KS_KoniecznyPinoPerez2008}, which are unique\footnote{Note that the notion of improvement operators is not used consistently in the literature. For instance, the improvement operators as defined in \cite{KS_KoniecznyGrespanPinoPerez2010} are not unique.}.
Again, we refer to the work of  Konieczny and Pino P{\'{e}}rez \cite{KS_KoniecznyPinoPerez2008} for these postulates, and only present the characterisation in the framework of total preorders.
\begin{proposition}[Improvement Operator{\cite[Thm. 2]{KS_KoniecznyPinoPerez2008}}]\label{prop:improve}
	A weak improvement operator $ \improve $ is an improvement operator if and only if there exists a strong faithful $ \Psi\mapsto\leq_{\Psi} $ assignment such that
	\begin{align*}
	& \ksIF \omega_1,\omega_2 \in \modelsOf{\alpha} \ksTHEN \omega_1 \leq_{\Psi} \omega_2 \Leftrightarrow \omega_1 \leq_{\Psi\improve\alpha} \omega_2 \tag{S1} \label{pstl:S1}                                       \\
	& \ksIF \omega_1,\omega_2 \in \modelsOf{\negOf{\alpha}} \ksTHEN \omega_1 \leq_{\Psi} \omega_2 \Leftrightarrow \omega_1 \leq_{\Psi\improve\alpha} \omega_2  \tag{S2} \label{pstl:S2}                              \\
	& \ksIF \omega_1\in\modelsOf{{\alpha}} \ksAND \omega_2\in\modelsOf{\negOf{\alpha}}  \ksTHEN    \omega_1 \leq_{\Psi} \omega_2 \Rightarrow \omega_1 <_{\Psi\improve\alpha} \omega_2      \tag{S3} \label{pstl:S3} \\
	& \ksIF \omega_1\in\modelsOf{{\alpha}} \ksAND \omega_2\in\modelsOf{\negOf{\alpha}}  \ksTHEN    \omega_1 <_{\Psi} \omega_2 \Rightarrow \omega_1 \leq_{\Psi\improve\alpha} \omega_2  \tag{S4} \label{pstl:S4}           \\
	& \ksIF \omega_1\in\modelsOf{{\alpha}} \ksAND \omega_2\in\modelsOf{\negOf{\alpha}}  \ksTHEN    \omega_2 \ll_{\Psi} \omega_1 \Rightarrow \omega_1 \leq_{\Psi\improve\alpha} \omega_2 \tag{S5} \label{pstl:S5}
	\end{align*}
	holds and the following is satisfied:
	\begin{equation*}
	\modelsOf{\Psi \improveLimit \alpha} = \min(\modelsOf{\alpha},\leq_{\Psi})
	\end{equation*}
\end{proposition}
In the following section we use the basic ideas of (weak) improvement operators as a starting point for developing the weak decrement operators.

\section{Weak Decrement Operators}
\label{sec:weak_disimprovements}

A property of a contraction operator $ - $ is that the success condition of contraction is instantaneously achieved, i.e., if $ \alpha $ is believed in a state ($ \alpha \in \beliefsOf{\Psi}  $) then after the contraction with $ \alpha $, it is not believed any more ($ \alpha \notin \beliefsOf{\Psi-\alpha} $).
As a generalisation, we define \hesitant\ contractions as operators who achieve the success condition of contraction after multiple consecutive applications.

\begin{definition}
A belief change operator $ \circ $ is called a \headword{\hesitant\ contraction operator} if the following postulate is fulfilled:
	\begin{align*}
& \ksIF \alpha\not\equiv\top \ksTHEN  
\text{there exists } n\in\naturals_0 \text{ such that } \alpha \notin \beliefsOf{ \Psi \circ^n \alpha}
 \tag{hesitance} \label{pstl:hesitant_success} 
\end{align*} 
\end{definition}
If  $ \circ $ is an \hesitant\ contraction operator, then we define a corresponding operator $ \bullet $ by $ \Psi \bullet \alpha = \Psi \circ^n \alpha $, where $ n=0 $ if $ \alpha\equiv\top $, otherwise $ n $ is the smallest integer such that $ \alpha\notin\beliefsOf{\Psi \circ^n \alpha} $.

The following Example \ref{exmpl:slow_love} shows a modelling application for \hesitant\ belief change operators.
\begin{example}\label{exmpl:slow_love}
Addison bought a new mobile with much easier handling. 
She does no longer have to press a sequence of buttons to access her favourite application. 
However, it takes multiple changes of her epistemic state before she contracts the belief of having to press the sequence of buttons for her favourite application.
\end{example}

We now introduce weak \decrement\ operators, which fulfil AGM-like contraction postulates, adapted for the decrement of beliefs.
\begin{definition}[Weak Decrement Operator]
	A belief change operator $ \circ $ is called a \headword{weak \decrement\ operator} if the following postulates are fulfilled:
	\begin{align*}
	& \beliefsOf{\Psi \bullet \alpha} \subseteq  \beliefsOf{\Psi} \tag{D1} \label{pstl:D1} \\
	& \ksIF \alpha \notin \beliefsOf{\Psi} \ksTHEN \beliefsOf{\Psi} \subseteq \beliefsOf{\Psi \bullet \alpha} \tag{D2} \label{pstl:D2} \\
		& \circ \text{ is a \hesitant\ contraction operator} \tag{D3} \label{pstl:D3} \\
	 &  \beliefsOf{\Psi} \subseteq  Cn(\beliefsOf{\Psi \bullet \alpha}\cup\{\alpha\}) \tag{D4} \label{pstl:D4}\\
	& \ksIF \alpha_1\equiv \beta_1, ..., \alpha_n\equiv \beta_n \ksTHEN	 \beliefsOf{\Psi\circ\alpha_1\circ...\circ\alpha_n}\! =\! \beliefsOf{\Psi\circ\beta_1\circ...\circ\beta_n} \tag{D5} \label{pstl:D5}\\
	& \beliefsOf{\Psi\bullet\alpha} \cap \beliefsOf{\Psi\bullet\beta} \subseteq \beliefsOf{\Psi\bullet(\alpha\land\beta)} \tag{D6} \label{pstl:D6}\\
	& \ksIF \beta\notin \beliefsOf{\Psi \bullet (\alpha\land\beta)} \ksTHEN \beliefsOf{\Psi \bullet (\alpha\land\beta)} \subseteq \beliefsOf{\Psi\bullet\beta} \tag{D7} \label{pstl:D7}
	\end{align*} 
\end{definition}
The postulates \eqref{pstl:D1} to \eqref{pstl:D7} correspond to the postulates \eqref{pstl:C-1} to \eqref{pstl:C-7}.
By \eqref{pstl:D1} a weak \decrement\ does not add new beliefs, and together with \eqref{pstl:D2} the beliefs of an agent are not changed if $ \alpha $ is not believed priorly.
\eqref{pstl:D3} ensures that after enough consecutive application a belief $ \alpha  $ is removed. 
\eqref{pstl:D4} is the recovery postulate, stating that removing $ \alpha $ and then adding $ \alpha $ again recovers all initial beliefs.
The postulate \eqref{pstl:D5} ensures syntax independence in the case of iteration.
\eqref{pstl:D6} and \eqref{pstl:D7} 
state that a contraction of a conjunctive belief is constrained by the results of the contractions with each of the conjuncts alone.

For the class of weak \decrement\ operators the following representation theorem holds:
\begin{theorem}[Representation Theorem: Weak Decrement Operators]\label{thm:weak_disimprovement}
	Let $ \circ $  be a belief change operator. Then the following items are equivalent:
	\begin{enumerate}[(a)]
		\item $ \circ $ is a weak \decrement\ operator 
		\item there exists a strong faithful assignment $ \Psi \mapsto \leq_\Psi $ with respect to $ \circ $ such that:
		\begin{align*}
	\tag{decrement sucess} \label{pstl:SFAdis}	\text{there exists } n\in\naturals_0  &\text{ such that } \modelsOf{\Psi \circ^n \alpha} = \modelsOf{\Psi} \cup \min(\modelsOf{\negOf{\alpha}}, \leq_{\Psi} )  \\
	& \ksAND n \text{ is the smallest integer such that } \modelsOf{\Psi \circ^n \alpha} \not\subseteq \modelsOf{\alpha}   \notag
		\end{align*}
	\end{enumerate}
\end{theorem}
From Theorem \ref{thm:weak_disimprovement} we easily get the following corollary:
	\begin{corollary}
		If $ \circ $ is a weak \decrement\ operator, then $ \bullet $ fulfils \eqref{pstl:C-1} to \eqref{pstl:C-7}.
		Furthermore, every belief change operator that fulfils \eqref{pstl:C-1} to \eqref{pstl:C-7} and \eqref{pstl:D5} is a weak \decrement\ operator. 
	\end{corollary}
This shows that weak \decrement\ operators are (up to \eqref{pstl:D5}) a generalisation of AGM contraction for epistemic states in the sense of Proposition \ref{prop:es_contraction}.

\section{Decrement Operators}
\label{sec:disimprovements}

We now introduce an ordering on the formulas in order to shorten our notion in the following postulates.
\begin{definition}\label{def:giveuprelation}
Let $ \circ $ be a  \hesitant\ contraction operator, then we define for every epistemic state $ \Psi $  and every two formula $ \alpha,\beta $:
\begin{equation*}
 \alpha	\preceq_\Psi^\circ \beta \text{ iff }   \beliefsOf{\Psi\bullet \alpha\beta} \subseteq  \beliefsOf{\Psi\bullet \alpha}
\end{equation*}
With $ \prec_\Psi^\circ  $ we denote the  strict variant of $ \preceq_\Psi^\circ $ and define $ \alpha \llcurly_\Psi^\circ \beta $ if $ \alpha \prec_\Psi^\circ \beta $ and there is no $ \gamma $ such that $ \alpha \prec_\Psi^\circ \gamma \prec_\Psi^\circ \beta$.
\end{definition}
Intuitively $ \alpha \prec_\Psi^\circ \beta $ means that in the state $ \Psi $ the agent is more willing to give up the belief $ \alpha $ than the belief $ \beta $.

For the iteration of decrement operators we give the following postulates:
	\begin{align*}
	& \ksIF 
	\negOf{\alpha} \models \beta
	\ksTHEN \beliefsOf{\Psi\circ\alpha\bullet\beta} =_\alpha \beliefsOf{\Psi\bullet\beta} \tag{D8} \label{pstl:D8} \\
	& \ksIF 
	\alpha\models\beta
	\ksTHEN \beliefsOf{\Psi\circ\alpha\bullet\beta} =_{\negOf{\beta}} \beliefsOf{\Psi\bullet\beta}  \tag{D9} \label{pstl:D9} \\
	&
	\ksIF 
	\alpha\models\gamma
	\ksTHEN \Psi \circ\alpha \bullet\beta \models \gamma \Rightarrow \Psi\bullet\beta \models \gamma \tag{D10} \label{pstl:D10} \\
	& \ksIF 
	\negOf{\alpha}\models\gamma
	\ksTHEN  \Psi\bullet\beta \models \gamma \ \Rightarrow\  \Psi\circ\alpha \bullet\beta \models \gamma \tag{D11} \label{pstl:D11}  \\
	& \ksIF \alpha\models\beta \ksAND \negOf{\alpha}\models\gamma \ksTHEN \gamma \llcurly_\Psi^\circ \beta \Rightarrow \beta \preceq^\circ_{\Psi\circ\alpha} \gamma \tag{D12} \label{pstl:D12} \\
	& \beliefsOf{\Psi\circ\alpha}\subseteq \beliefsOf{\Psi} \tag{D13} \label{pstl:D13} 
	\end{align*}
	\eqref{pstl:D8} states that a prior decrement with $ \alpha $ does not influence the beliefs of an decrement with $ \beta $ if $ \negOf{\alpha}\models\beta $.
	\eqref{pstl:D9} states that a prior decrement with $ \alpha $ does not influence the beliefs of an decrement with $ \beta $ if $ \alpha\models\beta $.
The postulate \eqref{pstl:D10} states that if a belief in $ \gamma $ is believed after a decrement of $ \alpha $ and the removal of $ \beta $, then only a removal of $ \beta $ does not influence the belief in $ \gamma
	 $ if $ \alpha\negOf{\gamma} $ implies $ \beta $.
	By \eqref{pstl:D11}, if $ \gamma $ and $ \alpha $ do not share anything, then a decrease of $ \alpha $ does not influence this belief. 
	By \eqref{pstl:D12}, if in the state $ \Psi $ the agent prefers removing a consequence of $ \negOf{\alpha} $ minimally more than removing a consequence of $ \alpha $, then after a decrement of $ \alpha $, she is more willing to remove the consequence of $ \alpha $.
	The postulate \eqref{pstl:D13} axiomatically enforces that a single step does not add beliefs.

	We call operators that fulfil these postulates \decrement\ operators.
\begin{definition}[Decrement Operator]\label{def:disimprovement}
	A $ \circ $ weak \decrement\  operator is called a \headword{\decrement\ operator} if $ \circ $ satisfies \eqref{pstl:D8} -- \eqref{pstl:D13}.
\end{definition}

On the semantic side, we define a specific form of strong faithful assignment which implements decrementing on total preorders.
\begin{definition}[Decreasing Assignment]\label{def:decreasing_assignment}
Let $ \circ $ be a \hesitant\ belief change operator. A strong faithful assignment $ \Psi\mapsto\leq_{\Psi} $ with respect to $ \circ $ is said to be a \headword{decreasing assignment} (with respect to $ \circ $) if the following postulates are satisfied:
	\begin{align*}
		 & \ksIF \omega_1,\omega_2 \in \modelsOf{\alpha} \ksTHEN \omega_1 \leq_{\Psi} \omega_2 \Leftrightarrow \omega_1 \leq_{\Psi\circ\alpha} \omega_2 \tag{DR8} \label{pstl:DR8}          \\
		 & \ksIF \omega_1,\omega_2 \in \modelsOf{\negOf{\alpha}} \ksTHEN \omega_1 \leq_{\Psi} \omega_2 \Leftrightarrow \omega_1 \leq_{\Psi\circ\alpha} \omega_2  \tag{DR9} \label{pstl:DR9} \\
		 & \ksIF \omega_1\in\modelsOf{\negOf{\alpha}} \ksAND \omega_2\in\modelsOf{\alpha}  \ksTHEN    \omega_1 \leq_{\Psi} \omega_2 \Rightarrow \omega_1 \leq_{\Psi\circ\alpha} \omega_2      \tag{DR10} \label{pstl:DR10}                  \\
		 & \ksIF \omega_1\in\modelsOf{\negOf{\alpha}} \ksAND \omega_2\in\modelsOf{\alpha}  \ksTHEN    \omega_1 <_{\Psi} \omega_2 \Rightarrow \omega_1 <_{\Psi\circ\alpha} \omega_2  \tag{DR11} \label{pstl:DR11}                  \\
		 & \ksIF \omega_1\in\modelsOf{\negOf{\alpha}} \ksAND \omega_2\in\modelsOf{\alpha}  \ksTHEN    \omega_2 \ll_{\Psi} \omega_1 \Rightarrow \omega_1 \leq_{\Psi\circ\alpha} \omega_2 \tag{DR12} \label{pstl:DR12} \\
 & \ksIF \omega_1\in\modelsOf{\negOf{\alpha}},\ \omega_2\in\modelsOf{\alpha} \ksAND \omega_2 \leq_{\Psi} \omega_3 \text{ for all } \omega_3 \ksTHEN    
 \omega_2 \leq_{\Psi\circ\alpha} \omega_1 \tag{DR13} \label{pstl:DR13}
	\end{align*}
\end{definition}
The postulates \eqref{pstl:DR8} to \eqref{pstl:DR11} are the same as given by Konieczny and {Pino P{\'{e}}rez} \cite{KS_KoniecznyPinoPerez2017} for iterated contraction (cf. Proposition \ref{prop:it_es_contraction}).
The postulate \eqref{pstl:DR12} states that a world of $ \negOf{\alpha} $ which is minimally less plausible than a world of $ \alpha $ should be made at least as plausible as this world of $ \alpha $. \eqref{pstl:DR13} ensures that (together with the other postulates) that world in $ \modelsOf{\Psi} $ stays plausible after a decrement.

The main result is that \decrement\ operators are exactly those which are compatible with a decreasing assignment.
\begin{theorem}[Representation Theorem: Decrement Operators]\label{thm:disimprovement}
	Let $ \circ $  be a belief change operator. Then the following items are equivalent:
\begin{enumerate}[(a)]
	\item $ \circ $ is a \decrement\ operator 
	\item there exists a decreasing assignment $ \Psi \mapsto \leq_\Psi $ with respect to $ \circ $ that satisfies \eqref{pstl:SFAdis}, i.e.:
\begin{align*}
\text{there exists } n\in\naturals_0  &\text{ such that } \modelsOf{\Psi \circ^n \alpha} = \modelsOf{\Psi} \cup \min(\modelsOf{\negOf{\alpha}}, \leq_{\Psi} ) \notag \\
& \ksAND n \text{ is the smallest integer such that } \modelsOf{\Psi \circ^n \alpha} \not\subseteq \modelsOf{\alpha}
\end{align*}
\end{enumerate}
\end{theorem}

The following proposition presents a nice property of \decrement\ operators:
Like AGM contraction for epistemic sates (cf. Proposition \ref{prop:es_contraction}) a decrement operators keeps plausible worlds; and only the least unplausible counter-worlds may become plausible.
\begin{proposition}\label{prop:decreasingassign_part_succ}
	Let $ \circ $ be a \hesitant\ belief change operator. If there exists a decreasing assignment $ \Psi\mapsto\leq_{\Psi} $ with respect to $ \circ $, then 
	we have:
	\begin{equation*}
	\modelsOf{\Psi}  \subseteq \modelsOf{\Psi\circ\alpha} \subseteq \modelsOf{\Psi} \cup \min(\modelsOf{\negOf{\alpha}},\leq_{\Psi}) 
	\tag{partial success} \label{pstl:SFApartSuccess}
	\end{equation*}
\end{proposition}

\section{Specific Decrement Operators}
\label{sec:stepwise_decrement_operator}
Unlike improvement operators \cite{KS_KoniecznyPinoPerez2008}, there is no unique decrement operator.
The reason for this is, that if $ \omega_2 \simeq_{\Psi} \omega_1 $ for $ \omega_1\in\modelsOf{\negOf{\alpha}} $ and $ \omega_2\in\modelsOf{\alpha} $, and it is not required otherwise by  \eqref{pstl:DR12}, 
then  the relative plausibility of $ \omega_1 $ and $ \omega_2 $ might not be changed by a \decrement\ operator $ \circ $, i.e. $ \omega_2 \simeq_{\Psi\circ\alpha} \omega_1.$
Example \ref{exmpl:decrement_freedom} demonstrates this.

\begin{table}[t]
	\begin{center}\setlength{\tabcolsep}{5pt}
		\begin{tabular}{@{}l|ll|ll|ll@{}}
			\toprule
			&    \multicolumn{2}{c}{$\Psi_1$}     & \multicolumn{2}{|c|}{$\Psi_1\circ_1 a$} &      \multicolumn{2}{c}{$\Psi_1\circ_2 a$}      \\ \midrule
			Layer 2             & $a\negOf{b}$ & $\negOf{a}\negOf{b}$ &              &                          & $a\negOf{b}$ &                                  \\
			\hdashline
			Layer 1       &              & $\negOf{a}b$         & $a\negOf{b}$ & $\negOf{a}\negOf{b}$     &              & $\negOf{a}\negOf{b}$             \\
			\hdashline
			Layer 0\ \ $ \modelsOf{\Psi} $\ \ & $ab$         &                      & $ab$         & $\negOf{a}b$             & $ab$         & \multicolumn{1}{c}{$\negOf{a}b$} \\ \bottomrule
		\end{tabular}
	\end{center}
	\caption{Example changes by two decrement operators $ \circ_1 $ and $ \circ_2 $.}\label{tab:example_stepwise}
\end{table}

\begin{example}\label{exmpl:decrement_freedom}
	Let $ \Sigma=\{ a,b \} $ and $ \Psi_1 $ be an epistemic state as given in Table \ref{tab:example_stepwise}.
	Then the change from $ \Psi_1 $ to $ \Psi_1\circ_2 a $ in Table \ref{tab:example_stepwise} is a valid change by a \decrement\  operator. 
	Likewise, the change from $ \Psi_1 $ to $ \Psi_1\circ_2 a $ from Table \ref{tab:example_stepwise} is also a valid change for a \decrement\ operator.
\end{example}

We capture this observation by two types of decrement operators.
In the first case, the \decrement\ operator improves the plausibility of a counter-model whenever it is possible.
\begin{definition}[Type-1 Decrement Operator]
A \decrement\ operator $ \stepwise $ is a \emph{type-1 \decrement\ operator} if there exists a decreasing assignment $ \Psi\mapsto\leq_{\Psi} $ with:
\begin{align*}
\ksIF \omega_1\in\modelsOf{\negOf{\alpha}} \ksAND \omega_2\in\modelsOf{\alpha}  \ksTHEN    \omega_2 \simeq_{\Psi} \omega_1 \Rightarrow \omega_1 \ll_{\Psi\stepwise\alpha} \omega_2  \tag{DR14} \label{pstl:DR14}                  
\end{align*}
\end{definition}

The second type of decrement operators keeps the order $ \omega_1\simeq_{\Psi} \omega_2 $ whenever possible.
We capture the cases when this is possible by the following notion.
If  $ \leq \subseteq \Omega\times\Omega $ is a total preorder on worlds, we say $ \omega_1 $ is \emph{frontal} with respect to $ \alpha $, if
(1.) there is no $ \omega_3\in \modelsOf{\alpha} $ such that $ \omega_3\ll \omega_1 $, and (2.) there is no $ \omega_3\in \modelsOf{\negOf{\alpha}} $ such that $ \omega_1\ll \omega_3 $.
We define the second type of decrement operators as follows.
\begin{definition}[Type-2 Decrement Operator]
	A \decrement\ operator $ \belated $ is a \emph{type-2 \decrement\ operator} if there exists a decreasing assignment $ \Psi\mapsto\leq_{\Psi} $ with:
	\begin{align*}
	\ksIF \omega_1\in\modelsOf{\negOf{\alpha}},\ \omega_2\in\modelsOf{\alpha} & \ksAND \omega_1 \text{ is frontal w.r.t } \alpha \ksTHEN    \omega_2 \simeq_{\Psi} \omega_1 \Rightarrow \omega_2 \simeq_{\Psi\belated\alpha} \omega_1   \tag{DR15} \label{pstl:DR15}                  
	\end{align*}
\end{definition}
\begingroup \noindent\textit{Example\  (continuation of Example \ref{exmpl:decrement_freedom})}.
	The change from $ \Psi_1 $ to $ \Psi_1\circ_1 a $ in Table \ref{tab:example_stepwise} can be made by a type-1 \decrement\ operator, but not by a type-2 \decrement\ operator. 
	Conversely, the change from $ \Psi_1 $ to $ \Psi_1\circ_2 a $ from Table \ref{tab:example_stepwise} can be made by a type-2 \decrement\ operator, but not by a type-1 \decrement\ operator
\endgroup

\section{Discussion and Future Work}
\label{sec:conclusion}
We provide postulates and representation theorems for gradual variants of AGM contractions in the Darwich-Pearl framework of epistemic states.
These so-called weak \decrement\ operators are a generalisation of AGM contraction for epistemic states.
Additionally, we give postulates for intended iterative behaviour of these operators, forming the class of \decrement\ operators.
For both classes of operators we presented a representation theorem in the framework of total preorders.
For the definition of the postulates, the new relation $ \preceq_\Psi^\circ $ (see Definition \ref{def:giveuprelation}) is introduced.
While $ \preceq_\Psi^\circ $ is related to epistemic entrenchment \cite{KS_Gaerdenfors1988}, it can be shown that $ \preceq_\Psi^\circ $ is not an epistemic entrenchment.
The exploration of the exact nature of $ \preceq_\Psi^\circ $ remains an open task.

The next natural step will be to investigate the interrelation between (weak) \decrement\ operators and (weak) improvement operators.
One approach is to generalize the Levi identity \cite{KS_Levi1977} and Haper identity \cite{KS_Harper1976} to these operators.
Another approach could be the direct definition of a contraction operator from improvement operators, as suggested by Konieczny and {Pino P{\'{e}}rez} \cite{KS_KoniecznyPinoPerez2008}.
For such operators, after achieving success,
a next improvement may make certain models
unplausible, while a decrement operator keeps the plausibility.
While this already indicated a difference between the operators, the study of their
specific interrelationship is part of future~work.
Another goal for future work is to generalize (weak) \decrement\ operators to a more general class of gradual change operators \cite{KS_Sauerwald2019}.
Such operators are candidates for a formalisation of psychologically inspired forgetting operations.
An immediate target towards this goal is to take a closer look at subclasses and interrelate them with the taxonomy of improvement operators \cite{KS_KoniecznyGrespanPinoPerez2010}.

\smallskip

\noindent\textbf{Acknowledgements:}
We thank the 
	reviewers for their valuable hints and comments that helped us to improve the paper and we thank Gabriele Kern-Isberner for fruitful discussions and her encouragement to follow the line of research leading to this paper.
This work was supported by DFG Grant BE 1700/9-1 given to
Christoph Beierle as part of the 
priority 
program "Intentional Forgetting in Organizations" (SPP 1921). Kai Sauerwald is supported by this Grant.

%
\bibliographystyle{plain}
\bibliography{bibexport}

\begin{thebibliography}{10}

\bibitem{KS_AlchourronGaerdenforsMakinson1985}
Carlos~E. Alchourr{\'{o}}n, Peter G{\"{a}}rdenfors, and David Makinson.
\newblock On the logic of theory change: Partial meet contraction and revision
  functions.
\newblock {\em J. Symb. Log.}, 50(2):510--530, 1985.

\bibitem{BeierleKernIsbernerSauerwaldBockRagni2019KIzeitschrift}
Christoph Beierle, Gabriele Kern{-}Isberner, Kai Sauerwald, Tanja Bock, and
  Marco Ragni.
\newblock Towards a general framework for kinds of forgetting in common-sense
  belief management.
\newblock {\em {KI -- K\"{u}nstliche Intelligenz}}, 33(1):57--68, 2019.

\bibitem{KS_BoothFermeKoniecznyPerez2014}
Richard Booth, Eduardo~L. Ferm{\'{e}}, S{\'{e}}bastien Konieczny, and
  Ram{\'{o}}n~Pino P{\'{e}}rez.
\newblock Credibility-limited improvement operators.
\newblock In Torsten Schaub, Gerhard Friedrich, and Barry O'Sullivan, editors,
  {\em {ECAI} 2014 - 21st European Conference on Artificial Intelligence, 18-22
  August 2014, Prague, Czech Republic}, volume 263 of {\em Frontiers in
  Artificial Intelligence and Applications}, pages 123--128. {IOS} Press, 2014.

\bibitem{KS_CaridroitKoniecznyMarquis2015}
Thomas Caridroit, S{\'{e}}bastien Konieczny, and Pierre Marquis.
\newblock Contraction in propositional logic.
\newblock In S{\'{e}}bastien Destercke and Thierry Denoeux, editors, {\em
  Symbolic and Quantitative Approaches to Reasoning with Uncertainty - 13th
  European Conference, {ECSQARU} 2015, Compi{\`{e}}gne, France, July 15-17,
  2015. Proceedings}, volume 9161 of {\em Lecture Notes in Computer Science},
  pages 186--196. Springer, 2015.

\bibitem{KS_CaridroitKoniecznyMarquis2017}
Thomas Caridroit, S{\'{e}}bastien Konieczny, and Pierre Marquis.
\newblock Contraction in propositional logic.
\newblock {\em Int. J. Approx. Reasoning}, 80:428--442, 2017.

\bibitem{KS_ChopraGhoseMeyerWong2008}
Samir Chopra, Aditya Ghose, Thomas~Andreas Meyer, and Ka{-}Shu Wong.
\newblock Iterated belief change and the recovery axiom.
\newblock {\em J. Philosophical Logic}, 37(5):501--520, 2008.

\bibitem{KS_DarwichePearl1997}
A.~Darwiche and J.~Pearl.
\newblock On the logic of iterated belief revision.
\newblock {\em Artificial Intelligence}, 89:1--29, 1997.

\bibitem{EiterKernIsberner2019KIzeitschrift}
Thomas Eiter and Gabriele Kern{-}Isberner.
\newblock A brief survey on forgetting from a knowledge representation and
  reasoning perspective.
\newblock {\em {KI -- K\"{u}nstliche Intelligenz}}, 33(1):9--33, 2019.

\bibitem{KS_FermeWassermann2018}
Eduardo Ferm{\'{e}} and Renata Wassermann.
\newblock On the logic of theory change: iteration of expansion.
\newblock {\em J. Braz. Comp. Soc.}, 24(1):8:1--8:9, 2018.

\bibitem{KS_Gaerdenfors1988}
Peter. Gärdenfors.
\newblock {\em Knowledge in flux : modeling the dynamics of epistemic states}.
\newblock MIT Press Cambridge, 1988.

\bibitem{KS_Harper1976}
William~L. Harper.
\newblock Rational conceptual change.
\newblock {\em PSA: Proceedings of the Biennial Meeting of the Philosophy of
  Science Association}, 1976:462--494, 1976.

\bibitem{KS_KatsunoMendelzon1992}
Hirofumi Katsuno and Alberto~O. Mendelzon.
\newblock Propositional knowledge base revision and minimal change.
\newblock {\em Artif. Intell.}, 52(3):263--294, 1992.

\bibitem{KS_KoniecznyGrespanPinoPerez2010}
S{\'{e}}bastien Konieczny, Mattia~Medina Grespan, and Ram{\'{o}}n {Pino
  P{\'{e}}rez}.
\newblock Taxonomy of improvement operators and the problem of minimal change.
\newblock In Fangzhen Lin, Ulrike Sattler, and Miroslaw Truszczynski, editors,
  {\em Principles of Knowledge Representation and Reasoning: Proceedings of the
  Twelfth International Conference, {KR} 2010, Toronto, Ontario, Canada, May
  9-13, 2010}. {AAAI} Press, 2010.

\bibitem{KS_KoniecznyPinoPerez2008}
S{\'{e}}bastien Konieczny and Ram{\'{o}}n {Pino P{\'{e}}rez}.
\newblock Improvement operators.
\newblock In Gerhard Brewka and J{\'{e}}r{\^{o}}me Lang, editors, {\em
  Principles of Knowledge Representation and Reasoning: Proceedings of the
  Eleventh International Conference, {KR} 2008, Sydney, Australia, September
  16-19, 2008}, pages 177--187. {AAAI} Press, 2008.

\bibitem{KS_KoniecznyPinoPerez2017}
S{\'{e}}bastien Konieczny and Ram{\'{o}}n {Pino P{\'{e}}rez}.
\newblock On iterated contraction: syntactic characterization, representation
  theorem and limitations of the {L}evi identity.
\newblock In {\em Scalable Uncertainty Management - 11th International
  Conference, {SUM} 2017, Granada, Spain, October 4-6, 2017, Proceedings},
  volume 10564 of {\em Lecture Notes in Artificial Intelligence}. Springer,
  2017.

\bibitem{KS_Levi1977}
Isaac Levi.
\newblock Subjunctives, dispositions and chances.
\newblock {\em Synthese}, 34(4):423--455, Apr 1977.

\bibitem{KS_Sauerwald2019}
Kai Sauerwald.
\newblock Student research abstract: Modelling the dynamics of forgetting and
  remembering by a system of belief changes.
\newblock In {\em The 34th ACM/SIGAPP Symposium on Applied Computing (SAC '19),
  April 8--12, 2019, Limassol, Cyprus}, New York, NY, USA, 2019. ACM.

\end{thebibliography}

\clearpage
\appendix
\pagenumbering{arabic}
\renewcommand{\thepage}{A.\arabic{page}}
\section{Proofs}
This appendix contains full proofs for the two representation theorems and for Proposition \ref{prop:decreasingassign_part_succ}.
These proofs rely on three lemmata which are also proven here.
\sloppy
\begin{lemma}\label{lem:contract_world}
	Let $ \circ $ an operator satisfying \eqref{pstl:D1} to \eqref{pstl:D4} and $ \omega\in\Omega $, then:
	\begin{equation*}
	\modelsOf{\Psi \bullet \negOf{\omega}} =  \modelsOf{\Psi} \cup\{\omega\}
	\end{equation*}
\end{lemma}
\begin{proof}
	The proof is analogue to a proof by Caridroit et. al \cite[Lem 13.]{KS_CaridroitKoniecznyMarquis2017}.\qed
\end{proof}

\setcounter{theorem}{0}
\begin{theorem}[Representation Theorem: Weak Decrement Operators]
Let $ \circ $  be a belief change operator. Then the following items are equivalent:
\begin{enumerate}[(a)]
	\item $ \circ $ is a weak \decrement\ operator 
	\item there exists a strong faithful assignment $ \Psi \mapsto \leq_\Psi $ with respect to $ \circ $ such that:
	\begin{align*}
	\text{there exists } n\in\naturals_0  &\text{ such that } \modelsOf{\Psi \circ^n \alpha} = \modelsOf{\Psi} \cup \min(\modelsOf{\negOf{\alpha}}, \leq_{\Psi} ) \notag \\
	& \ksAND n \text{ is the smallest integer such that } \modelsOf{\Psi \circ^n \alpha} \not\subseteq \modelsOf{\alpha}   \tag{decrement success} 
	\end{align*}
\end{enumerate}
\end{theorem}

\begin{proof}
	We proof the theorem under the assumption that the signature has more than 2 elements, i.e. $ |\Sigma|>2 $.
	For the (a) to (b)-direction, $ \circ $ is an \hesitant\ contraction operator, and the corresponding operator $ \bullet $ is defined. We define the total preorder  $ \leq_\Psi $ as follows:
	\begin{equation*}
	\omega_1 \leq_\Psi  \omega_2 \text{ iff } \omega_1\in\modelsOf{\Psi\bullet \negOf{(\omega_1\lor\omega_2)}}
	\end{equation*}
	We show that $ \leq_{\Psi} $ is a total preorder:
	\begin{description}
		\item[Totality] Let $ \omega_1,\omega_2\in\Omega $. By definition $ \modelsOf{(\omega_1\lor\omega_2)}=\{\omega_1,\omega_2\} $, and therefore $ \negOf{(\omega_1\lor\omega_2)} $ has at least one model and $ \negOf{(\omega_1\lor\omega_2)} \not\equiv \top $. By \eqref{pstl:hesitant_success} there is an $ n $ (and we choose here the smallest) such that $ \negOf{(\omega_1\lor\omega_2)} \notin \beliefsOf{\Psi \bullet \negOf{(\omega_1\lor\omega_2)}} $. Therefore, either $ \omega_1\in\modelsOf{\Psi \bullet \negOf{(\omega_1\lor\omega_2)}} $ or $ \omega_2 \in \modelsOf{\Psi \bullet \negOf{(\omega_1\lor\omega_2)}} $.
		\item[Reflexivity] Follows from totality.
		\item[Transitivity] Let $ \omega_1,\omega_2,\omega_3\in\Omega $ such that $ \omega_1 \leq_{\Psi} \omega_2 $ and $ \omega_2 \leq_{\Psi} \omega_3 $. We differentiate by case:
		\begin{itemize}
			\item If $ \omega_1,\omega_2,\omega_3 $ are not pairwise distinct, then transitivity is easily fulfilled (since $ \leq_{\Psi} $ is reflexive).
			\item Assume that $ \omega_1,\omega_2,\omega_3 $ are pairwise distinct and for at least one $ 1\leq i\leq 3 $ we have $ \omega_i\in\modelsOf{\Psi} $. Then in each case it is easy to see that $ \omega_1\in\modelsOf{\Psi} $ and thus, by \eqref{pstl:D1}, for all $ \alpha $ it follows $ \omega_1\in \modelsOf{\Psi\circ \alpha} $. 
			\item Assume that $ \omega_1,\omega_2,\omega_3 $ are pairwise distinct and $ \omega_1,\omega_2,\omega_3\notin \modelsOf{\Psi} $.
			Towards a contradiction, assume that $ \omega_1 \not\leq_{\Psi} \omega_3 $.  
			By assumption of $ \omega_1 \not\leq_{\Psi} \omega_3 $  we have $ \omega_1\notin \modelsOf{\Psi \bullet \negOf{(\omega_1\lor\omega_3)} } $. By Lemma \ref{lem:contract_world} we have $ \omega_1\in \modelsOf{\Psi}\cup\{\omega_1\} =\modelsOf{\Psi\bullet\neg\omega_1} $. From \eqref{pstl:D7} we get $ \modelsOf{\Psi\bullet\neg\omega_1} \subseteq \modelsOf{\Psi \bullet (\negOf{\omega_1} \land \negOf{\omega_3})} $ and by \eqref{pstl:D5} we have $ \modelsOf{\Psi \bullet (\negOf{\omega_1} \land \negOf{\omega_3})} = \modelsOf{\Psi \bullet \negOf{(\omega_1\lor\omega_3)}} $, a contradiction, since $ \omega_1\in \modelsOf{\Psi \bullet \negOf{(\omega_1\lor\omega_3)}} $ and $ \omega_1\notin \modelsOf{\Psi \bullet \negOf{(\omega_1\lor\omega_3)}} $.
		\end{itemize}
	\end{description}
	We show that $ \leq_{\Psi} $ is a strong faithful assignment with respect to $ \circ $.
	\begin{description}
		\item[\eqref{pstl:SFA1}] Let $ \omega_1,\omega_2\in\modelsOf{\Psi} $. Then by \eqref{pstl:D1} we have $ \omega_1,\omega_2\in \modelsOf{\Psi\bullet \negOf{(\omega_1\lor\omega_2)}} $. Therefore by definition of $ \leq_{\Psi} $ we have $ \omega_1 \simeq_\Psi \omega_2 $.
		\item[\eqref{pstl:SFA2}] Let $ \omega_1\in\modelsOf{\Psi} $ and $ \omega_2\notin \modelsOf{\Psi} $. Then by \eqref{pstl:D1} we have $ \omega_1\in \modelsOf{\Psi\bullet \negOf{(\omega_1\lor\omega_2)}} $. 
		Towards a contradiction assume $ \omega_2 \in \modelsOf{\Psi\bullet \negOf{(\omega_1\lor\omega_2)}} $. Since $ \omega_1\neq \omega_2 $, by $ \omega_1\in \modelsOf{\Psi} $ we know that $ \Psi\not\models\negOf{(\omega_1\lor\omega_2)} $. Thus, by \eqref{pstl:D2} we have $ \modelsOf{\Psi \bullet \negOf{(\omega_1\lor\omega_2)} } \subseteq \modelsOf{\Psi} $.
		\item[\eqref{pstl:SFA3}] Follows directly from \eqref{pstl:D5}.
	\end{description}
	We show that \eqref{pstl:SFAdis} is fulfilled. We differentiate by case:
	\begin{itemize}
		\item Case with $ \alpha\equiv\top $. Then $ \modelsOf{\negOf{\alpha}}=\emptyset $ and by definition of $ \bullet $ we have $ \Psi=\Psi\bullet\alpha $, especially $ \modelsOf{\Psi} = \modelsOf{\Psi\bullet\alpha} = \modelsOf{\Psi}\cup \modelsOf{\negOf{\alpha}} $.
		\item Case with $ \alpha\notin\beliefsOf{\Psi} $. Then by \eqref{pstl:D1} and \eqref{pstl:D2} we have $ \modelsOf{\Psi} = \modelsOf{\Psi\bullet\alpha} $, resp. $ \beliefsOf{\Psi} = \beliefsOf{\Psi\bullet\alpha} $. Then there is an $ \omega\in \modelsOf{\Psi} $ such that $ \omega\not\models\alpha $, thus, we have $ \min(\modelsOf{\negOf{\alpha}},\leq_{\Psi}) \subseteq \modelsOf{\Psi} $. We conclude  $ \modelsOf{\Psi\bullet\alpha} = \modelsOf{\Psi} \cup \min(\modelsOf{\negOf{\alpha}},\leq_{\Psi}) $.
		\item Case with $ \alpha\in\beliefsOf{\Psi} $. Then by \eqref{pstl:D1} we have $ \modelsOf{\Psi}\subseteq \modelsOf{\Psi\bullet\alpha} $. We show that every $ \omega\in \modelsOf{\Psi\bullet\alpha}\setminus \modelsOf{\Psi} $ is an element of the set $ \min(\modelsOf{\negOf{\alpha}},\leq_{\Psi}) $. 
		
		First, by \eqref{pstl:D4} we have $ \modelsOf{\Psi\bullet\alpha}\cap\modelsOf{\alpha} \subseteq   \modelsOf{\Psi} $. Then every $ \omega\in\modelsOf{\alpha} $ which is an element of $ \modelsOf{\Psi\bullet\alpha}\setminus\modelsOf{\Psi} $ leads to a violation of \eqref{pstl:D4}. Thus, we observe that every $ \omega\in \modelsOf{\Psi\bullet\alpha}\setminus \modelsOf{\Psi}  $ is an element of $ \modelsOf{\negOf{\alpha}} $.
		
		Second, towards a contradiction suppose $ \omega\in \modelsOf{\Psi\bullet\alpha}\setminus \modelsOf{\Psi} $ such that $ \omega\notin \min(\modelsOf{\negOf{\alpha}},\leq_{\Psi}) $. Let $ \omega' \in \min(\modelsOf{\negOf{\alpha}},\leq_{\Psi})  $, and therefore $ \omega' <_{\Psi} \omega $. By definition we have $ \omega' \in \modelsOf{\Psi\bullet \negOf{(\omega'\lor\omega)}}  $ and $ \omega \notin \modelsOf{\Psi\bullet \negOf{(\omega'\lor\omega)}}  $. By \eqref{pstl:D5} we have $ \modelsOf{\Psi\bullet \negOf{(\omega'\lor\omega)}}=\modelsOf{\Psi\bullet (\negOf{\omega'}\land\negOf{\omega})}$.
		Then by \eqref{pstl:D7} and by Lemma \ref{lem:contract_world} we conclude $ \modelsOf{\Psi}\cup\{ \omega \} \subseteq \modelsOf{\Psi\bullet (\negOf{\omega'}\land\negOf{\omega})} $.
		This shows $ \modelsOf{\Psi\bullet\alpha} \subseteq \modelsOf{\Psi} \cup \min(\modelsOf{\negOf{\alpha}},\leq_{\Psi})  $.
		
		Suppose $ \omega$ is an element of $ \min(\modelsOf{\negOf{\alpha}},\leq_{\Psi}) $ such that $ \omega\notin\modelsOf{\Psi\bullet\alpha} $. Without loss of generality we can assume $ \alpha\not\equiv\top $; thus, there exists at least one $ \omega'\in \min(\modelsOf{\negOf{\alpha}},\leq_{\Psi})$ such that $ \omega'\in\modelsOf{\Psi\bullet\alpha} $. By definition of $ \leq_{\Psi} $ we have $ \omega\in \modelsOf{\Psi\bullet \negOf{(\omega\lor\omega')} } $.
		Clearly $ \neg\alpha=\gamma \lor \omega \lor \omega' $, and thus, $ \alpha \equiv \negOf{\gamma} \land \negOf{(\omega \lor \omega')} $. Since $ \omega'\in \modelsOf{\Psi\bullet\alpha} $, we have $ \Psi\bullet \alpha\not\models \negOf{(\omega \lor \omega')} $. Therefore from \eqref{pstl:D7} we conclude $ \modelsOf{\Psi \bullet  \negOf{(\omega \lor \omega')}} \subseteq \modelsOf{\Psi\bullet\alpha} $ and thus the contradiction $ \omega\in\modelsOf{\Psi\bullet\alpha} $.
		This completes the proof of $ \modelsOf{\Psi\bullet\alpha} = \modelsOf{\Psi} \cup \min(\modelsOf{\negOf{\alpha}},\leq_{\Psi}) $.
	\end{itemize}
	
	For the (b) to (a)-direction let $ \circ $ be a belief change operator and $ \leq_{\Psi} $ a strong faithful assignment with respect to $ \circ $ such that \eqref{pstl:SFAdis} is fulfilled.
	\begin{description}
		\item[\eqref{pstl:D3}] For $ \Psi $ and $ \alpha $ let $ n^\Psi_\alpha $ be the smallest integer such  that $ \modelsOf{\Psi \circ^{n^\Psi_\alpha} \alpha} = \modelsOf{\Psi} \cup \min(\modelsOf{\negOf{\alpha}}, \leq_{\Psi} ) $. By \eqref{pstl:SFAdis} the existence of $ n^\Psi_\alpha $ guaranteed. For $ \alpha\not\equiv\top $, then $ \alpha\notin\beliefsOf{\Psi\circ^{n^\Psi_\alpha} \alpha} $ and therefore $ \circ $ is a \hesitant\ contraction operator.
	\end{description}
	Since $ \circ $ satisfies \eqref{pstl:D3} the corresponding operator $ \bullet $ is defined.
	\begin{description}
		\item[\eqref{pstl:D1}] Follows directly by \eqref{pstl:SFAdis}.
		\item[\eqref{pstl:D2}] Suppose $ \Psi\not\models\alpha $. Then, $ \min(\modelsOf{\negOf{\alpha}},\leq_{\Psi}) $ is non-empty and by \eqref{pstl:SFA1} and \eqref{pstl:SFA2} we have $ \min(\modelsOf{\negOf{\alpha}},\leq_{\Psi}) \subseteq \modelsOf{\Psi} $.
		\item[\eqref{pstl:D4}] Let $ \gamma \in \beliefsOf{\Psi} $ and therefore $ \modelsOf{\Psi} \subseteq \modelsOf{\gamma} $. 
			Then $ \gamma\in Cn(\beliefsOf{\Psi\bullet\alpha}\cup \{ \alpha \}) $ if and only if $ \modelsOf{\Psi\bullet\alpha}\cap\modelsOf{\alpha} \subseteq \modelsOf{\gamma} $.
		By \eqref{pstl:SFAdis} we conclude $ \modelsOf{\Psi\bullet\alpha}\cap\modelsOf{\alpha} = (\modelsOf{\Psi} \cup \min(\modelsOf{\negOf{\alpha}},\leq_{\Psi})) \cap \modelsOf{\alpha} = \modelsOf{\Psi}\setminus \modelsOf{\negOf{\alpha}}  $.
		Clearly, $ \modelsOf{\Psi}\setminus \modelsOf{\negOf{\alpha}} \subseteq \modelsOf{\Psi} \subseteq \modelsOf{\gamma}  $.
		\item[\eqref{pstl:D5}] Follows by \eqref{pstl:SFA3}.
		\item[\eqref{pstl:D6}] By \eqref{pstl:SFAdis} we have $ \modelsOf{\Psi \circ (\alpha\land\beta)} = \modelsOf{\Psi} \cup \min(\modelsOf{\negOf{\alpha}} \cup \modelsOf{\negOf{\beta}},\leq_{\Psi}) $ and we have $ \modelsOf{\Psi \circ \alpha} \cup \modelsOf{\Psi \circ \beta} = \modelsOf{\Psi} \cup \min(\modelsOf{\negOf{\alpha}},\leq_{\Psi}) \cup \min( \modelsOf{\negOf{\beta}},\leq_{\Psi}) $.
		Furthermore, it holds that $ \min(\modelsOf{\negOf{\alpha}} \cup \modelsOf{\negOf{\beta}},\leq_{\Psi}) \subseteq \min(\modelsOf{\negOf{\alpha}},\leq_{\Psi}) \cup \min( \modelsOf{\negOf{\beta}},\leq_{\Psi})  $ and therefore, we have:
		\begin{equation*}
		\modelsOf{\Psi \circ (\alpha\land\beta)}  \subseteq  \modelsOf{\Psi} \cup  \min(\modelsOf{\negOf{\alpha}},\leq_{\Psi}) \cup \min( \modelsOf{\negOf{\beta}},\leq_{\Psi}) = \modelsOf{\Psi \circ \alpha} \cup \modelsOf{\Psi \circ \beta}
		\end{equation*}
		\item[\eqref{pstl:D7}] Assume $ \Psi\bullet\alpha\beta \not\models\beta $. Then by \eqref{pstl:SFAdis} and \eqref{pstl:SFA3} we have $ \modelsOf{\Psi\bullet\alpha\beta}= \modelsOf{\Psi} \cup \min(\modelsOf{\negOf{\alpha}\lor\negOf{\beta}},\leq_{\Psi}) \not\subseteq \modelsOf{\beta}  $. 
		This implies that $ {\min(\modelsOf{\negOf{\beta}},\leq_{\Psi})}  \subseteq  \min(\modelsOf{\negOf{\alpha}\lor\negOf{\beta}},\leq_{\Psi})  $.
		By basic set theory we get $ {\modelsOf{\Psi} \cup \min(\modelsOf{\negOf{\beta}},\leq_{\Psi})}  \subseteq \modelsOf{\Psi} \cup \min(\modelsOf{\negOf{\alpha}\lor\negOf{\beta}},\leq_{\Psi}) $.
		By \eqref{pstl:SFAdis} this is equivalent to $ \modelsOf{\Psi\bullet\beta} \subseteq \modelsOf{\Psi\bullet\alpha\beta} $ 
	\end{description}
 In summary, the operator $ \circ $ is an weak \decrement\ operator. \qed
\end{proof}

\begin{lemma}\label{lem:dr_part_succ}
Let $ \circ $ be a belief change operator. If there exists a strong faithful assignment $ \Psi\mapsto\leq_{\Psi} $ with respect to $ \circ $ which satisfies \eqref{pstl:DR8}, \eqref{pstl:DR9} and \eqref{pstl:DR11}, then for every $ \Psi $ and $ \alpha\in\propLang $ we have:
\begin{equation*}
\modelsOf{\Psi\circ\alpha} \subseteq \modelsOf{\Psi} \cup \min(\modelsOf{\negOf{\alpha}},\leq_{\Psi}) 
\end{equation*}
\end{lemma}
\begin{proof}
	Let $ \omega\in \modelsOf{\Psi\circ\alpha} $. If $ \omega\in\modelsOf{\Psi} $ we are done, so it remains to show that $ \omega \in \min(\modelsOf{\negOf{\alpha}},\leq_{\Psi})  $ in the case of $ \omega\notin \modelsOf{\Psi} $.
	
	We first show that if $ \omega\notin \modelsOf{\Psi} $, then $ \omega\in\modelsOf{\negOf{\alpha}} $.
	Towards a contradiction suppose this is not the case, i.e. $ \omega\notin \modelsOf{\Psi} $  and $ \omega\in\modelsOf{\alpha} $.
	Then there a two cases:
	1. There exists $ \omega'\in\modelsOf{\alpha} $ such that $ \omega' \in \modelsOf{\Psi} $. 
	We easy conclude that $ \omega' <_\Psi \omega $ and thus, by \eqref{pstl:DR8}, we have $ \omega' <_{\Psi\circ\alpha} \omega $. 
	Due to the faithfulness of the assignment $ \omega \notin \modelsOf{\Psi\circ\alpha} $, which is a contradiction.
	2. For all $ \omega'\in\modelsOf{\alpha} $ we have $ \omega'\notin\modelsOf{\Psi} $. 
	Then, by using $ \modelsOf{\Psi}\neq\emptyset $, for all $ \omega''\in\modelsOf{\Psi} $ we must have $ \omega''\in\modelsOf{\negOf{\Psi}} $. Thus, $ \omega'' <_\Psi \omega $ and from \eqref{pstl:DR11} we get $ \omega''<_{\Psi\circ\alpha} \omega $.
	Again, due to the faithfulness of the assignment, we have $ \omega \notin \modelsOf{\Psi\circ\alpha} $, which is a contradiction.
	So every $ \omega\in\modelsOf{\Psi\circ\alpha}\setminus\modelsOf{\Psi} $ is an element of $ \omega\in\modelsOf{\negOf{\alpha}} $.
	
	Now we show that every $ \omega\in\modelsOf{\Psi\circ\alpha}\setminus\modelsOf{\Psi} $ is an element of $ \min(\modelsOf{\negOf{\alpha}},\leq_{\Psi}) $.
	Towards a contradiction suppose $ \omega\in\modelsOf{\negOf{\alpha}}\setminus\min(\modelsOf{\negOf{\alpha}},\leq_{\Psi}) $.
	Then there exists $ \omega'\in \min(\modelsOf{\negOf{\alpha}},\leq_{\Psi}) $ such that $ \omega' <_{\Psi\circ\alpha} \omega $.
	By \eqref{pstl:DR9} we can conclude that $ \omega' <_{\Psi\circ\alpha} \omega $, which is a contradiction to the assumed faithfulness of the assignment.
	\qed
\end{proof}

\setcounter{proposition}{4}
\begin{proposition}
	Let $ \circ $ be a \hesitant\ belief change operator. If there exists an decreasing assignment $ \Psi\mapsto\leq_{\Psi} $ with respect to $ \circ $, then 
	we have:
	\begin{equation*}
	\modelsOf{\Psi}  \subseteq \modelsOf{\Psi\circ\alpha} \subseteq \modelsOf{\Psi} \cup \min(\modelsOf{\negOf{\alpha}},\leq_{\Psi}) 
	\tag{partial success} 
	\end{equation*}
\end{proposition}
\begin{proof}
	This is a direct consequence of Lemma \ref{lem:dr_part_succ} and \eqref{pstl:DR13}.\qed
\end{proof}

\begin{lemma}\label{lem:order_formula}
	Let $ \circ $ be a belief change operator, $ \Psi\mapsto\leq_{\Psi} $ a strong faithful assignment with respect to $ \circ $ and $ \gamma \prec_\Psi^\circ \beta $. Then $ \gamma \llcurly_\Psi^\circ \beta $ if and only if for each $ \omega_1\in\min(\modelsOf{\negOf{\beta}},\leq_{\Psi}) $ and $ \omega_2\in\min(\modelsOf{\negOf{\gamma}},\leq_{\Psi}) $ we have either $ \omega_2 \ll_{\Psi} \omega_1 $ or $ \omega_2 \simeq_\Psi \omega_1 $.
\end{lemma}
\begin{proof}
	The "only if" direction.
	By definition of $ \preceq_\Psi^\circ $ we have 
	\begin{equation*}
	\min(\modelsOf{\negOf{\gamma}},\leq_{\Psi}) \subseteq \min(\modelsOf{\negOf{\gamma}\lor\negOf{\beta}},\leq_{\Psi}),
	\end{equation*} which implies $ \min(\modelsOf{\negOf{\gamma}},\leq_{\Psi}) \subseteq \min(\modelsOf{\negOf{\beta}},\leq_{\Psi}) $. Clearly, it follows that $ \omega_2 \leq_{\Psi} \omega_1 $.
	In the case of $ \omega_2 \simeq_\Psi \omega_1 $ we are done. 
	
	For the remaining case of $ \omega_2 <_\Psi \omega_1 $ suppose there exists $ \omega_3\notin\{\omega_1,\omega_2\} $ such that $ \omega_2 <_\Psi \omega_3 <_\Psi \omega_1 $.
	This implies that $ \min(\modelsOf{\negOf{\gamma}},\leq_{\Psi}) \subseteq \min(\modelsOf{\negOf{\gamma}\lor\omega_3},\leq_{\Psi}) $ and $ \omega_3\notin \min(\modelsOf{\negOf{\gamma}},\leq_{\Psi}) $. 
	Thus, by definition we have $ \gamma \prec_\Psi^\circ \gamma\negOf{\omega_3} $.
	Similarly, we have $ \beta\negOf{\omega_3} \prec_\Psi^\circ \beta $ , since $ \omega_3\in \min(\modelsOf{\negOf{\beta}\lor\omega_3},\leq_{\Psi}) $ and $ \min(\modelsOf{\negOf{\beta}},\leq_{\Psi}) \not\subseteq \min(\modelsOf{\negOf{\gamma}\lor\omega_3},\leq_{\Psi}) $. Note that this implies $ \omega_3\not\models\negOf{\beta} $.
	From the previous observations we conclude  $ \min(\modelsOf{\negOf{\gamma}},\leq_{\Psi}) \subseteq \min(\modelsOf{\negOf{\gamma}\lor\negOf{\beta}\lor\omega_3},\leq_{\Psi}) $, and therefore $ \gamma \preceq_\Psi^\circ \beta\negOf{\omega_3} $.
	This leads to $ \gamma \prec_\Psi^\circ \beta\negOf{\omega_3} \prec_\Psi^\circ \beta $, which is a contradiction to $ \gamma \llcurly_\Psi^\circ \beta $.
	In summary it must be the case that either $ \omega_2 \simeq_\Psi \omega_1 $ or $ \omega_2 \ll_{\Psi} \omega_1 $.
	
	For the "if" direction suppose that $ \gamma\prec_\psi^\circ \alpha \prec_\Psi^\circ \beta $.
	This implies that
	$ \min(\modelsOf{\negOf{\gamma}},\leq_{\Psi}) \subseteq \min(\modelsOf{\negOf{\gamma} \lor \negOf{\alpha}},\leq_{\Psi})  $ and $ \min(\modelsOf{\negOf{\gamma}},\leq_{\Psi}) \not\subseteq {\min(\modelsOf{\negOf{\gamma} \lor \negOf{\alpha}},\leq_{\Psi})} $.
	Thus we have $ \omega_2 <_\Psi \omega_3 $ for every $ \omega_2\in \min(\modelsOf{\negOf{\gamma}},\leq_{\Psi}) $ and some $ \omega_3\in \min(\modelsOf{\negOf{\alpha}},\leq_{\Psi}) $.
	Additionally, we have $ \min(\modelsOf{\negOf{\alpha}},\leq_{\Psi}) \subseteq \min(\modelsOf{\negOf{\beta} \lor \negOf{\alpha}},\leq_{\Psi})  $ and $ \min(\modelsOf{\negOf{\beta}},\leq_{\Psi}) \not\subseteq \min(\modelsOf{\negOf{\beta} \lor \negOf{\alpha}},\leq_{\Psi}) $.
	Thus we have $ \omega_4 <_\Psi \omega_1 $ for every $ \omega_4\in \min(\modelsOf{\negOf{\gamma}},\leq_{\Psi}) $ and some $ \omega_1\in \min(\modelsOf{\negOf{\alpha}},\leq_{\Psi}) $.
	Note that $ \leq_{\Psi} $ is a total preorder, and thus, we have $ \omega_2 <_\Psi \omega_3 <_\Psi \omega_1 $, a contradiction to the assumptions of $ \omega_2 \ll_{\Psi} \omega_1 $ or $ \omega_2\simeq_{\Psi} \omega_1 $.
	\qed
\end{proof}

\begin{theorem}[Representation Theorem: Decrement Operators]
	Let $ \circ $  be a belief change operator. Then the following items are equivalent:
\begin{enumerate}[(a)]
	\item $ \circ $ is a \decrement\ operator 
	\item there exists a decreasing assignment $ \Psi \mapsto \leq_\Psi $ with respect to $ \circ $ that satisfies: \eqref{pstl:SFAdis}, i.e.:
	\begin{align*}
	\text{there exists } n\in\naturals_0  &\text{ such that } \modelsOf{\Psi \circ^n \alpha} = \modelsOf{\Psi} \cup \min(\modelsOf{\negOf{\alpha}}, \leq_{\Psi} ) \notag \\
	& \ksAND n \text{ is the smallest integer such that } \modelsOf{\Psi \circ^n \alpha} \not\subseteq \modelsOf{\alpha}  
	\end{align*}
\end{enumerate}
\end{theorem}
\begin{proof}
	\noindent (a) to (b)-direction: As $ \circ $ is an \hesitant\ contraction operator, the corresponding operator $ \bullet $ is defined. We define the total preorder  $ \leq_\Psi $ as follows:
	\begin{equation*}
	\omega_1 \leq_\Psi \omega_2 \text{ iff } \omega_1\in\modelsOf{\Psi\bullet \negOf{(\omega_1\lor\omega_2)}}
	\end{equation*}
	By Theorem \ref{thm:weak_disimprovement} (and its proof) $ \leq_{\Psi} $ is a strong faithful assignment with respect to $ \circ $ which satisfies \eqref{pstl:SFAdis}. We show the satisfaction of \eqref{pstl:DR8} to \eqref{pstl:DR12}.
	\begin{description}
		\item[\eqref{pstl:DR8}] Let $ \omega_1,\omega_2\in\modelsOf{\alpha} $. Choose $ \beta=\negOf{(\omega_1\lor\omega_2)} $ and therefore $ \negOf{\beta}\models\alpha $.
By \eqref{pstl:D8} we have $ \beliefsOf{\Psi\change\alpha\change\beta}=_\alpha\beliefsOf{\Psi\change\beta} $, which implies:
	\begin{equation}
	\modelsOf{\Psi\change\beta} =_\alpha \modelsOf{\Psi\change\alpha\bullet\beta} \label{eq:ual:DR8}
	\end{equation}
	From \eqref{pstl:SFAdis} we obtain
	\begin{align}
	\modelsOf{\Psi\change\alpha\bullet\beta} & =\modelsOf{\Psi\change\alpha}\cup\min(\modelsOf{\negOf{\beta}},\leq_{\Psi\change\alpha})                                  \label{eq:DR8:2} 
	\end{align} and
	\begin{equation}
	\modelsOf{\Psi\bullet\beta}=\modelsOf{\Psi}\cup\min(\modelsOf{\negOf{\beta}},\leq_{\Psi}) . \label{eq:DR8:3}
	\end{equation}
	Substituting \eqref{eq:DR8:2} and \eqref{eq:DR8:3}  into Equation \eqref{eq:ual:DR8} leads to 
	\begin{equation*}
	\modelsOf{\Psi\change\alpha}\cup\min(\modelsOf{\negOf{\beta}},\leq_{\Psi\change\alpha}) =_\alpha 
	\modelsOf{\Psi}\cup\min(\modelsOf{\negOf{\beta}},\leq_{\Psi}) .
	\end{equation*}
	Now consider two cases:
	\begin{itemize}
		\item Suppose $ \modelsOf{\Psi\change\alpha}\cap\modelsOf{\negOf{\beta}}=\emptyset $. Due to the faithfulness of the assignment, we conclude $ \min(\modelsOf{\negOf{\beta}},\leq_{\Psi\change\alpha}) =_{\negOf{\beta}} 
		\modelsOf{\Psi}\cup\min(\modelsOf{\negOf{\beta}},\leq_{\Psi}) $.
		\item For $ \modelsOf{\Psi\change\alpha}\cap\modelsOf{\negOf{\beta}}\neq\emptyset $, from the faithfulness of the assignment we get $ \modelsOf{\Psi\change\alpha}\cap\modelsOf{\negOf{\beta}}=\min(\modelsOf{\negOf{\beta}},\leq_{\Psi\change\alpha}) $.
		Then again, we conclude $ \min(\modelsOf{\negOf{\beta}},\leq_{\Psi\change\alpha}) =_{\negOf{\beta}} 
		\modelsOf{\Psi}\cup\min(\modelsOf{\negOf{\beta}},\leq_{\Psi}) $.
	\end{itemize}
	In particular, we can conclude from both cases that:
	\begin{equation}
	\min(\modelsOf{\negOf{\beta}},\leq_{\Psi\change\alpha}) = \min(\modelsOf{\negOf{\beta}},\leq_{\Psi}) \label{eq:DR8:4}
	\end{equation}
	Note that $ \modelsOf{\negOf{\beta}} $ has only two elements, $ \modelsOf{\negOf{\beta}}=\{\omega_1,\omega_2\} \subseteq \modelsOf{\alpha} $, and thus information about the minima provides us the relative order of the two elements $ \omega_1$ and $\omega_2 $. So, from Equation \eqref{eq:DR8:4}, we can conclude that $ \omega_1 \leq_\Psi \omega_2 $ if and only if $ \omega_1 \leq_{\Psi\change\alpha} \omega_2 $.
		\item[\eqref{pstl:DR9}] 
		Suppose $ \omega_1,\omega_2\in\modelsOf{\negOf{\alpha}} $. We choose $ \beta=\negOf{(\omega_1\lor\omega_2)} $ and therefore, we have $ \negOf{\beta}\models\negOf{\alpha} $.
			By \eqref{pstl:D9} we have $ \beliefsOf{\Psi\change\alpha\bullet\beta}=_\negOf{\beta}\beliefsOf{\Psi\bullet\beta} $, which implies:
			\begin{equation}
			\modelsOf{\Psi\change\beta} =_\negOf{\beta} \modelsOf{\Psi\change\alpha\bullet\beta} \label{eq:ual:DR9:1}
			\end{equation}
			From \eqref{pstl:SFAdis} we obtain
			\begin{align}
			\modelsOf{\Psi\change\alpha\bullet\beta} & =\modelsOf{\Psi\change\alpha}\cup\min(\modelsOf{\negOf{\beta}},\leq_{\Psi\change\alpha})                                  \label{eq:DR9:2} 
			\end{align} and
			\begin{equation}
			\modelsOf{\Psi\bullet\beta}=\modelsOf{\Psi}\cup\min(\modelsOf{\negOf{\beta}},\leq_{\Psi}) . \label{eq:DR9:3}
			\end{equation}
			Substituting \eqref{eq:DR9:2} and \eqref{eq:DR9:3}  into Equation \eqref{eq:ual:DR9:1} leads to 
			\begin{equation}
			\modelsOf{\Psi\change\alpha}\cup\min(\modelsOf{\negOf{\beta}},\leq_{\Psi\change\alpha}) =_\negOf{\beta} 
			\modelsOf{\Psi}\cup\min(\modelsOf{\negOf{\beta}},\leq_{\Psi}) . \label{eq:ual:DR9:4}
			\end{equation}			
			Now consider two cases:
			\begin{itemize}
				\item Suppose $ \modelsOf{\Psi\change\alpha}\cap\modelsOf{\negOf{\beta}}=\emptyset $. 
				Due to the faithfulness of the assignment, we conclude $ \min(\modelsOf{\negOf{\beta}},\leq_{\Psi\change\alpha}) =_{\negOf{\beta}} 
				\modelsOf{\Psi}\cup\min(\modelsOf{\negOf{\beta}},\leq_{\Psi}) $.
				\item For $ \modelsOf{\Psi\change\alpha}\cap\modelsOf{\negOf{\beta}}\neq\emptyset $, from the faithfulness of the assignment we get $ \modelsOf{\Psi\change\alpha}\cap\modelsOf{\negOf{\beta}}=\min(\modelsOf{\negOf{\beta}},\leq_{\Psi\change\alpha}) $.
				Then again, we conclude $ \min(\modelsOf{\negOf{\beta}},\leq_{\Psi\change\alpha}) =_{\negOf{\beta}} 
				\modelsOf{\Psi}\cup\min(\modelsOf{\negOf{\beta}},\leq_{\Psi}) $.
			\end{itemize}
			
			In both cases we can conclude:
			\begin{equation}
			\min(\modelsOf{\negOf{\beta}},\leq_{\Psi\change\alpha}) = \min(\modelsOf{\negOf{\beta}},\leq_{\Psi}) \label{eq:ual:CR9:7}
			\end{equation}
			Note that $ \modelsOf{\negOf{\beta}} $ has only two elements, $ \modelsOf{\negOf{\beta}}=\{\omega_1,\omega_2\} \subseteq \modelsOf{\negOf{\alpha}} $, and thus information about the minima provides us the relative order of the two elements $ \omega_1$ and $\omega_2 $. So, from Equation \eqref{eq:ual:CR9:7}, we can conclude that $ \omega_1 \leq_\Psi \omega_2 $ if and only if $ \omega_1 \leq_{\Psi\change\alpha} \omega_2 $.
		\item[\eqref{pstl:DR10}] 
		First, observe that the proof of satisfaction of
		\eqref{pstl:DR8},	\eqref{pstl:DR9}, \eqref{pstl:DR11} and \eqref{pstl:DR13} are independent from showing \eqref{pstl:DR10}, and hence we can safely assume their satisfaction.
		By Lemma \ref{lem:dr_part_succ}, $ \circ $ fulfils \eqref{pstl:SFApartSuccess}, i.e.:
		\begin{equation*}
		\modelsOf{\Psi} \subseteq \modelsOf{\Psi\circ\alpha} \subseteq \modelsOf{\Psi} \cup \min(\modelsOf{\negOf{\alpha}},\leq_{\Psi}) 
		\end{equation*}
		Let $ \omega_1\in\modelsOf{\negOf{\alpha}} $ and $ \omega_2\in\modelsOf{\alpha} $ and $ \omega_2 <_{\Psi\circ\alpha} \omega_1 $ and $ \beta=\negOf{(\omega_1\lor\omega_2)} $. Then $ \omega_2\in \modelsOf{\Psi\circ\alpha\bullet\beta} $ and $ \omega_1\notin \modelsOf{\Psi\circ\alpha\bullet\beta} $. 
		We show that $ \omega_2 <_\Psi \omega_1 $. This is the case if $ \omega_1\not\in\modelsOf{\Psi\bullet\beta} $ and $ \omega_2\in\modelsOf{\Psi\bullet\beta} $.
		By \eqref{pstl:SFApartSuccess} $ \omega_1$ is not an element of $\modelsOf{\Psi} $. Since $ \modelsOf{\Psi\bullet\beta}=\modelsOf{\Psi} \cup \min(\modelsOf{\negOf{\beta}},\leq_{\Psi}) $, it remains to show that $ \min(\modelsOf{\negOf{\beta}},\leq_{\Psi})=\{\omega_2\}  $.
		We have two cases:
		\begin{enumerate}[1.]
			\item For $ \min(\modelsOf{\negOf{\beta}},\leq_\Psi)=\{\omega_2\} $ we conclude directly $ \omega_2 <_\Psi \omega_1 $.
			\item Now consider the case of $ \omega_1\in\min(\modelsOf{\negOf{\beta}},\leq_\Psi) $, and therefore $ \omega_1\in\modelsOf{\Psi\bullet\beta} $. 
			Let $ \gamma=\gamma'\lor\alpha $, where $ \gamma' $ is a formula such that $ \modelsOf{\gamma}=\modelsOf{\Psi\circ\alpha\bullet\beta} $. Observe now that $ 
			\omega_1\not\models\gamma $ and that we have chosen $ \gamma $ and $ \beta $ such that $ \alpha\models\gamma $.
			From $ \Psi\circ\alpha\bullet\beta\models\gamma $ we conclude $ \Psi\bullet\beta\models\gamma $ by \eqref{pstl:D10}, a contradiction to $ \omega_1\in\modelsOf{\Psi\bullet\beta} $.
		\end{enumerate}
		In summary, it must be the case that $ \omega_2 <_\Psi \omega_1 $ and thus, we have shown the satisfaction of \eqref{pstl:DR10}.
		\item[\eqref{pstl:DR11}] 
			Suppose $ \omega_1\in\modelsOf{\negOf{\alpha}} $, $ \omega_2\in\modelsOf{{\alpha}} $ and $ \omega_1 <_{\Psi}\omega_2 $. We want to show $ \omega_1 <_{\Psi\change\alpha} \omega_2 $.
		For this purpose let $ \beta=\negOf{(\omega_1\lor\omega_2)} $. 
		Since $ \Psi\mapsto\leq_{\Psi} $ is a faithful assignment it must be the case that $ \omega_2\notin \modelsOf{\Psi} $.
		By use of \eqref{pstl:SFAdis} we can conclude that $ \omega_2\notin\modelsOf{\Psi\bullet\beta} $ and $ \omega_1\in\modelsOf{\Psi\bullet\beta} $.
		Now let $ \gamma=\gamma'\lor\negOf{\alpha} $, where $ \gamma' $ is a formula such that $ \modelsOf{\Psi\bullet\beta}\cup\{ \omega_1 \}= \modelsOf{\gamma} $.
			Note that $ \negOf{\alpha}\models\gamma $ and $ \omega_2\not\models\gamma $.
			By using
			\eqref{pstl:D11} we conclude $ \Psi\change\alpha\bullet\beta\models \gamma $.
			This implies that $ \omega_2\notin \modelsOf{\Psi\change\alpha\bullet\beta} $. 
			Note that by $ \modelsOf{\negOf{\beta}}=\{\omega_1,\omega_2\} $ and \eqref{pstl:SFAdis} it must be the case that $ \omega_1\in \modelsOf{\Psi\change\alpha\bullet\beta} $ or $ \omega_2\in \modelsOf{\Psi\change\alpha\bullet\beta} $,
			leaving the only option $ \omega_1\in \modelsOf{\Psi\change\alpha\bullet\beta} $.
			In summary, we get $ \omega_1 <_{\Psi\change\alpha} \omega_2 $.
		\item[\eqref{pstl:DR12}]
Let $ \omega_2\ll_{\Psi}\omega_1 $ with $ \omega_2\models\alpha  $ and $ \omega_1\models\negOf{\alpha} $.
This means $ \omega_2<_\Psi \omega_1 $ and there exists no $ \omega_3 $ such that $ \omega_2 <_\Psi \omega_3 <_\Psi \omega_1 $.
To show that $ \omega_1\leq_{\Psi\circ\alpha} \omega_2 $, let $ \gamma=\negOf{\omega_2}\lor\negOf{\alpha} $ and $ \beta=\negOf{\omega_1}\lor\alpha. $.
Then, we have $ \negOf{\alpha}\models\gamma $ and $ \alpha\models\beta $, and
\begin{align*}
\min(\modelsOf{\omega_2\alpha},\leq_\Psi)=\min(\modelsOf{\omega_2},\leq_\Psi) & \subseteq \min(\modelsOf{\omega_2\lor\omega_1},\leq_\Psi) \\
\min(\modelsOf{\omega_1\negOf{\alpha}},\leq_\Psi)=\min(\modelsOf{\omega_1},\leq_\Psi) & \not\subseteq \min(\modelsOf{\omega_2\lor\omega_1},\leq_\Psi).
\end{align*}
Clearly, this is equivalent to $ \omega_2 \prec_\Psi \omega_1 $.
Moreover, by Lemma \ref{lem:order_formula} we have $ \omega_2 \llcurly_\Psi \omega_1 $, and thus by \eqref{pstl:D12}, we have $ \omega_1 \preceq_{\Psi\circ\alpha} \omega_2 $.
By definition we have
\begin{align*}
\min(\modelsOf{\omega_1},\leq_\Psi) & \subseteq \min(\modelsOf{\omega_2\lor\omega_1},\leq_\Psi)
\end{align*} 
from which it is easy to conclude that $ \omega_1 \leq_{\Psi\circ\alpha} \omega_2 $.
	\item[\eqref{pstl:DR13}]
	Let $ \omega_1\modelsOf{\negOf{\alpha}} $ and $ \omega_2\in\modelsOf{\alpha} $ such that $ \omega_2\leq_{\Psi}\omega_3  $ for all $ \omega_3 $.
	Then $ \omega_2\in\modelsOf{\Psi} $ and thus $ \omega_2\in\modelsOf{\Psi\circ\alpha} $ by \eqref{pstl:D13}.
	Clearly, then we have $ \omega_2\leq_{\Psi\circ\alpha} \omega_1 $.
	\end{description}
	
	\noindent (b) to (a)-direction: Suppose that $ \Psi\mapsto\leq_{\Psi} $ is a decreasing assignment with respect to $ \circ $. By Theorem \ref{thm:weak_disimprovement}, the belief change operator $ \circ $ fulfils \eqref{pstl:D1} -- \eqref{pstl:D7}. We show the satisfaction of
	\eqref{pstl:D8} to \eqref{pstl:D13}.
	\begin{description}	
		\item[\eqref{pstl:D8}] 
		Let $ \negOf{\beta}\models\alpha $. By \eqref{pstl:SFAdis} we have to show $ \modelsOf{\Psi} \cup {\min(\modelsOf{\negOf{\beta}},\leq_{\Psi})} =_\alpha \modelsOf{\Psi\circ\alpha} \cup \min(\modelsOf{\negOf{\beta}},\leq_{\Psi\circ\alpha}) $.
		Then, by assumption, Lemma \ref{lem:dr_part_succ} and \eqref{pstl:DR8} and \eqref{pstl:DR13}, it is easy to see that $ \min(\modelsOf{\negOf{\beta}},\leq_{\Psi}) = \min(\modelsOf{\negOf{\beta}},\leq_{\Psi\circ\alpha}) $.
		\item[\eqref{pstl:D9}]
		Let $ \negOf{\beta} \models \negOf{\alpha} $. By \eqref{pstl:SFAdis} we have to show $ \modelsOf{\Psi} \cup {\min(\modelsOf{\negOf{\beta}},\leq_{\Psi})} = \modelsOf{\Psi\circ\alpha} \cup \min(\modelsOf{\negOf{\beta}},\leq_{\Psi\circ\alpha}) $.
			Then, by assumption, Lemma \ref{lem:dr_part_succ} and \eqref{pstl:DR9}, it is easy to see that $ \min(\modelsOf{\negOf{\beta}},\leq_{\Psi}) = \min(\modelsOf{\negOf{\beta}},\leq_{\Psi\circ\alpha}) $.
		\item[\eqref{pstl:D10}] 
		Let 
		$  \alpha \models \gamma $
		and $ \Psi\circ\alpha\bullet\beta\models\gamma $.
			By Proposition \eqref{prop:decreasingassign_part_succ} we conclude $ \modelsOf{\Psi} \subseteq \modelsOf{\gamma} $ and $ \min(\modelsOf{\negOf{\beta}},\leq_{\Psi\circ\alpha}) \subseteq  \modelsOf{\gamma} $.
			Remember that $ \circ $ satisfies \eqref{pstl:SFAdis} and therefore, $ \modelsOf{\Psi\bullet\beta} = \modelsOf{\Psi} \cup \min(\modelsOf{\negOf{\beta}},\leq_{\Psi}) $.
			Now let $ \omega_1\in\modelsOf{\negOf{\beta}} $ such that $ \omega_1 \notin \min(\modelsOf{\negOf{\beta}},\leq_{\Psi\circ\alpha}) $. We show that $ \omega_1\notin \min(\modelsOf{\negOf{\beta}},\leq_{\Psi}) $ or $ \omega_1\models\gamma $. Let $ \omega_2\in \min(\modelsOf{\negOf{\beta}},\leq_{\Psi\circ\alpha}) $ and thus, $ \omega_2 <_{\Psi\circ\alpha} \omega_1 $. We differentiate by case:
			\begin{enumerate}
				\item For $ \omega_1\in\modelsOf{\alpha} $ we conclude $ \omega_1\models\gamma $ directly from $ \alpha\models\gamma $.
				\item In the case of $ \omega_1\in\modelsOf{\negOf{\alpha}} $ and $ \omega_2\in\modelsOf{\alpha} $ we conclude $ \omega_2 <_{\Psi} \omega_1 $ by contraposition of \eqref{pstl:DR10}.
				\item In the remaining case of $ \omega_1,\omega_2\in\modelsOf{\negOf{\alpha}} $ it is easy to conclude by \eqref{pstl:DR9} that $ \omega_2 <_{\Psi} \omega_1 $.
			\end{enumerate}
			This shows that either $ \omega_2 <_\Psi \omega_1 $ or $ \omega_1\models\gamma $, leading to the conclusion that $ \min(\modelsOf{\negOf{\beta}},\leq_{\Psi})\subseteq \modelsOf{\gamma} $.
			In summary we have $ \modelsOf{\Psi\bullet\beta} = \modelsOf{\Psi} \cup {\min(\modelsOf{\negOf{\beta}},\leq_{\Psi})} \subseteq \modelsOf{\gamma} $.
		\item[\eqref{pstl:D11}] 
		Let
$ \negOf{\alpha}\models\gamma $
			and $ \Psi\bullet\beta\models\gamma $. 
			We want to show $ \Psi\change\alpha\bullet\beta\models\gamma $.
			By satisfaction of \eqref{pstl:SFAdis} we have
			\begin{equation}
			\modelsOf{\Psi\bullet\beta} = \modelsOf{\Psi} \cup \min( \modelsOf{\negOf{\beta}} , \leq_{\Psi} ) \subseteq \modelsOf{\gamma}
			\label{eq:proof:D11:1}
			\end{equation}
			and by Lemma \ref{lem:dr_part_succ} and \eqref{pstl:SFAdis} we have
			\begin{align}
			\modelsOf{\Psi\change\alpha\bullet\beta} & = \modelsOf{\Psi\change\alpha} \cup \min( \modelsOf{\negOf{\beta}} , \leq_{\Psi\change\alpha} ), 
			\label{eq:proof:D11:2a}\\
			\modelsOf{\Psi\change\alpha\bullet\beta} & \subseteq \modelsOf{\Psi} \cup \min( \modelsOf{\negOf{\alpha}} , \leq_{\Psi} ) \cup \min( \modelsOf{\negOf{\beta}} , \leq_{\Psi\change\alpha} )
			 .
			 \label{eq:proof:D11:2b}
			\end{align}
			We show that every $ \omega \in \modelsOf{\Psi\change\alpha\bullet\beta} $ is a model of $ \gamma $.
			\begin{itemize}
				\item If $ \omega\in\modelsOf{\Psi} $, then by Equation \eqref{eq:proof:D11:1} we have $ \omega\models\gamma $.
				\item For $ \omega \in \min( \modelsOf{\negOf{\beta}} , \leq_{\Psi\change\alpha} ) $ assume that $ \omega\models\negOf{\gamma} $.
				For $ \omega\models\negOf{\alpha} $, we directly conclude $ \omega\models\gamma $ from $ \negOf{\alpha}\models\gamma $.  Therefore we can assume $ \omega\models\alpha $.
				Since $ \min( \modelsOf{\negOf{\beta}} , \leq_{\Psi} ) \subseteq \modelsOf{\gamma} $, there must be $ \omega_1\in \min( \modelsOf{\negOf{\beta}} , \leq_{\Psi} ) $ such that $ \omega_1 <_\Psi \omega $.
				If $ \omega_1,\omega \in \modelsOf{\alpha} $, then $ \omega_1 <_{\Psi\change\alpha} \omega $ by \eqref{pstl:DR8}.
				For $ \omega_1 \in \modelsOf{\alpha} $ and $ \omega \in \modelsOf{\negOf{\alpha}} $ we conclude $ \omega_1 <_{\Psi\change\alpha} \omega $ by \eqref{pstl:DR11}.
				Thus it must be the case that $ \omega_1 <_{\Psi\change\alpha} \omega $, which is a contradiction to the minimality of $ \omega $.
				\item Suppose that $ \omega \in \min( \modelsOf{\negOf{\alpha}} , \leq_{\Psi\change\alpha} ) $.
				Then, $ \omega\models\gamma $ can be directly obtained from $ \negOf{\alpha}\models\gamma $. 
			\end{itemize}
			From Equation \eqref{eq:proof:D11:2b}  it follows that $ \omega \models \gamma $, and therefore $ \Psi\change\alpha\bullet\beta\models\gamma $.
		\item[\eqref{pstl:D12}]
		Let $ \alpha\models\beta $ and $ \negOf{\alpha}\models\gamma $, and $ \gamma \llcurly_\Psi \beta $.
		We show now that $ \beta \preceq_{\Psi\circ\alpha} \gamma $, which is the case when $ \min(\modelsOf{\negOf{\beta}},\leq_{\Psi\circ\alpha}) \subseteq \min(\modelsOf{\negOf{\beta}\lor \negOf{\gamma}},\leq_{\Psi\circ\alpha}) $.
	
		First, observe that $ \min(\modelsOf{\negOf{\gamma}},\leq_{\Psi\circ\alpha}) \subseteq \modelsOf{\alpha} $ and $ \min(\modelsOf{\negOf{\beta}},\leq_{\Psi\circ\alpha}) \subseteq \modelsOf{\negOf{\alpha}} $. 
		Thus for every $ \omega_2\in\min(\modelsOf{\negOf{\gamma}},\leq_{\Psi}) $ and every $ \omega_1 \in \min(\modelsOf{\negOf{\beta}},\leq_{\Psi}) $ we have $ 
		\omega_2\in\modelsOf{\alpha} $ and $ \omega_1\in\modelsOf{\negOf{\alpha}} $.
		Therefore, by $ \gamma \llcurly_\Psi \beta $ and Lemma \ref{lem:order_formula} we have two cases:
		\begin{itemize}
			\item In the case of $ \omega_2\llcurly_\Psi \omega_1 $ we conclude by \eqref{pstl:DR12} that $ \omega_1 \leq_{\Psi\circ\alpha} \omega_2 $.
			\item In the case of $ \omega_2\simeq_{\Psi} \omega_1 $ we have $  \omega_1 \leq_{\Psi} \omega_2 $, and hence, by \eqref{pstl:DR10}, we have $ \omega_1\leq_{\Psi\circ\alpha} \omega_2 $.
		\end{itemize}
		From \eqref{pstl:DR8} and \eqref{pstl:DR9} we get
		$	\min(\modelsOf{\negOf{\gamma}},\leq_{\Psi}) =  \min(\modelsOf{\negOf{\gamma}},\leq_{\Psi\circ\alpha}) $ and
		$	\min(\modelsOf{\negOf{\beta}},\leq_{\Psi}) =   \min(\modelsOf{\negOf{\beta}},\leq_{\Psi\circ\alpha})$.
		In summary, we have $ \min(\modelsOf{\negOf{\beta}},\leq_{\Psi\circ\alpha}) \subseteq \min(\modelsOf{\negOf{\beta}\lor\negOf{\gamma}},\leq_{\Psi\circ\alpha}) $, which is equivalent to $ \beta \preceq_{\Psi\circ\alpha} \gamma $.
		\item[\eqref{pstl:D13}]
		Let $ \omega\in \modelsOf{\Psi} $. If $ \omega\in\modelsOf{\alpha} $, then by \eqref{pstl:DR13} and \eqref{pstl:DR8} we have $ \omega\in\modelsOf{\Psi\circ\alpha} $.
		In the case of $ \omega\in\modelsOf{\negOf{\alpha}} $ we have $ \omega\in\modelsOf{\Psi\circ\alpha} $ by \eqref{pstl:DR9}  and \eqref{pstl:DR10}.
	\end{description}
In summary, $ \Psi\mapsto\leq_{\Psi} $ is a decreasing assignment.
\qed
\end{proof}

\end{document}